\documentclass[runningheads]{llncs}
\usepackage[T1]{fontenc}
\usepackage{graphicx}
\usepackage[hidelinks]{hyperref}

\usepackage{color}

\usepackage{algorithm}
\usepackage{algorithmic}

\usepackage{amsmath,amssymb,paralist}
\usepackage[capitalize]{cleveref}
\usepackage{booktabs, multirow}
\usepackage{graphicx}
\usepackage{caption}
\usepackage{float} 
\usepackage{tikz}
\usetikzlibrary{arrows,arrows.meta, automata, positioning, shapes.geometric}
\usepackage{pifont}
\usepackage{xspace}
\usepackage{textcomp}
\usepackage{thmtools}
\usepackage{makecell}

\usepackage{longtable}
\usepackage{booktabs}
\usepackage{multirow}
\usepackage{adjustbox}

\usepackage{array}
\newsavebox{\dummy}
\usepackage{siunitx}
\newcolumntype{d}{S[round-mode=places,round-precision=3,table-format=4.3]}

\usepackage{tcolorbox}
\tcbuselibrary{breakable}

\newtcolorbox{rexample}{
  breakable,
  colback=gray!10,
  colframe=black,
  boxrule=0pt,        
  toprule=1pt,      
  bottomrule=1pt,
  bottomrule at break=0pt,
  toprule at break=0pt,
  rounded corners,
}

\date{}

\usepackage{subcaption}
\usepackage{todonotes}
\usepackage{soul}
\usepackage{wrapfig}
\usepackage{orcidlink}

\newcommand{\distribution}{\mathsf{Dist}}
\newcommand\del[1]{}
\newcommand\ao{\ensuremath\mathcal{O}^A}
\newcommand\mco{\ensuremath\mathcal{O}^{M}}

\newsavebox{\fmbox}
\newenvironment{fmpage}[1]
     {\begin{center}\begin{lrbox}{\fmbox}\begin{minipage}{#1}}
     {\end{minipage}\end{lrbox}\fbox{\usebox{\fmbox}}\end{center}}

\newcommand{\mdpstates}{S}
\newcommand{\mdpinitstates}{s_0}
\newcommand{\mdpactions}{A}

\newcommand{\mdptransitions}{P}
\newcommand{\mdppath}{\rho}
\newcommand{\mdpstrategy}{\sigma}

\newcommand{\mc}{\mathcal{M}}
\newcommand{\mcstates}{S}
\newcommand{\mctrans}{P}
\newcommand{\mcinitstates}{s_0}

\newcommand{\cylinder}{\mathsf{Cyl}}

\newcommand{\pomdp}{\mathcal{P}}
\newcommand{\observationmap}{O}
\newcommand{\observations}{Z}
\newcommand{\belief}{\mathcal{B}}

\newcommand{\history}{h}
\newcommand{\histories}{\mathsf{Hist}_{\pomdp}}

\newcommand{\discount}{\lambda}
\newcommand{\mdpreward}{r}

\newcommand{\DRew}{\textsf{Rew}}

\newcommand{\fsc}{\mathcal{F}}
\newcommand{\fscstates}{N}
\newcommand{\fscactionmap}{\gamma}
\newcommand{\fscinitstate}{n_0}
\newcommand{\fsctransitionmap}{\delta}

\newcommand{\hypothesis}{\mathcal{H}}

\newcommand{\lstartab}{\ObsTable}

\DeclareDocumentCommand{\post}{D<>{} O{} D(){}}{\mathsf{Post}_{#1}^{#2}\ifthenelse{\isempty{#3}}{}{(#3)}}

\renewcommand{\Pr}{\mathbb{P}}

\renewcommand{\path}{\rho}

\renewcommand{\S}{\mathsf{S}}
\newcommand{\E}{\mathsf{E}}
\newcommand{\T}{\mathsf{T}}

\newcommand{\until}{\;\mathcal{U}\,}
\newcommand{\threshold}{\alpha}

\newcommand{\supp}{\mathsf{supp}}

\newcommand{\N}{\mathbb{N}}

\newcommand{\despot}{\textsc{Despot}\xspace}

\newcommand{\bad}{\mathsf{Bad}}
\newcommand{\good}{\mathsf{Good}}

\newcommand{\cplus}{\textsc{Cplus}\xspace}
\newcommand{\cplusds}{\textsc{Cplus}(\textsc{D}, \textsc{S})}
\newcommand{\cplusss}{\textsc{Cplus}(\textsc{S}, \textsc{S})}
\newcommand{\paynt}{\textsc{Paynt}\xspace}
\newcommand{\storm}{\textsc{Storm}\xspace}
\newcommand{\aalpy}{\textsc{Aalpy}\xspace}

\newcommand{\tsp}{\mathcal{G}}
\newcommand{\prph}{\Pr_{\pomdp\times\hypothesis}}
\newcommand{\prps}{\mathbb{P}_{\pomdp}^{\sigma}}

\newcommand{\expsp}{\mathbb{E}_{\pomdpr}^{\sigma}}
\newcommand{\cex}{\mathcal{R}}

\newcommand{\init}{o_0}
\newcommand{\memb}{\textsf{AQ}}
\renewcommand{\equiv}{\textsf{MCQ}}
\newcommand{\ce}{\mathcal{R}}

\newcommand{\MCsafety}{\Pr_{\mc}(\Box\neg\bad)>\threshold}
\newcommand{\Hyposafety}{\Pr_{\pomdp\times\hypothesis}(\Box\neg\bad)>\threshold}

\newcommand{\dc}{\texttt{x}}
\newcommand{\dct}{\ \texttt{x}\ }
\newcommand{\ObsTable}{\mathcal{T}}
\crefname{problem}{Problem}{}

\newcommand{\pomdpr}{\pomdp^r}

\newcommand{\redcell}{\mathbin{\,\textcolor{red}{\blacksquare}\,}}
\newcommand{\graycell}{\mathbin{\,\textcolor{gray}{\blacksquare}\,}}
\newcommand{\bluecell}{\mathbin{\,\textcolor{blue}{\blacksquare}\,}}

\newcommand{\ua}{\ \uparrow\ }
\newcommand{\da}{\ \downarrow\ }
\newcommand{\ra}{\rightarrow~}

\newif\ifhidecolumn
\hidecolumnfalse  
\newcolumntype{H}{>{\ifhidecolumn\begin{lrbox}{\dummy}\fi}c<{\ifhidecolumn\end{lrbox}\fi}@{}}

\title{Synthesizing POMDP Policies: Sampling Meets Model-checking via Learning}

\author{%
Debraj Chakraborty\inst{1}\orcidlink{0000-0003-0978-4457} \and 
Anirban Majumdar\inst{2}\orcidlink{0000-0003-4793-1892} \and
Prince Mathew\inst{3}\orcidlink{0000-0001-6410-1474} \and \\
Sayan Mukherjee\inst{4}\orcidlink{0000-0001-6473-3172} \and 
Jean-Fran\c{c}ois Raskin\inst{3}\orcidlink{0000-0002-3673-1097}
}%
\institute{
Nanyang Technological University, Singapore\\ \email{debraj.chakraborty@ntu.edu.sg} \and
Tata Institute of Fundamental Research, Mumbai, India\\
\email{anirban.majumdar@tifr.res.in} \and
Universit\'{e} Libre de Bruxelles, Brussels, Belgium\\
\email{\{prince.mathew,jean-francois.raskin\}@ulb.be}\and
 IITB Trust Lab, Department of CSE, IIT Bombay, Mumbai, India\\
 \email{sayan@cse.iitb.ac.in}
}

\titlerunning{Synthesizing POMDP Policies}
\authorrunning{Chakraborty et al.}

\begin{document}

\maketitle
  
\begin{abstract}
Partially Observable Markov Decision Processes (POMDPs) are the standard framework for decision-making under uncertainty. While sampling-based methods scale well, they lack formal correctness guarantees, making them unsuitable for safety-critical applications. Conversely, formal synthesis techniques provide correctness-by-construction but often struggle with scalability, as general POMDP synthesis is undecidable. To bridge this gap, we propose a synthesis framework that integrates sampling, automata learning, and model-checking. Inspired by Angluin’s $L^*$ algorithm, our approach utilizes sampling as a membership oracle and model-checking as an equivalence oracle. This enables the synthesis of finite-state controllers with formal guarantees, provided the sampling-induced policy is regular. We establish a relative completeness result for this framework. Experimental results from our prototypical implementation demonstrate that this method successfully solves threshold-safety problems that remain challenging for existing formal synthesis tools. We believe our algorithm serves as a valuable component in a portfolio approach to tackling the inherent difficulty of POMDP synthesis problems.
\end{abstract}

\section{Introduction}

Partially Observable Markov Decision Processes (POMDPs) are a fundamental model for decision-making under uncertainty and partial observability. They formalize a wide range of real-world control problems in which the environment evolves stochastically, and the controller receives only indirect, noisy, or incomplete observations of the underlying state. As a result, they are extensively studied in fields such as robotics, planning, and artificial intelligence, and offer a powerful framework for formulating the reactive controller synthesis problem in complex and uncertain environments.

Despite their expressive power and practical importance, tool support for controller synthesis in POMDPs remains limited. This is primarily due to the inherent algorithmic complexity of the model. Even the most basic decision problems related to synthesis -- existence of a controller that ensures safety (\emph{i.e.}, avoiding bad states) with a probability above a given threshold -- are undecidable in general~\cite{madani1999undecidability}, ruling out the possibility of complete algorithmic solutions. As a consequence, most practical approaches rely on sampling-based techniques, such as Monte Carlo Tree Search (MCTS)~\cite{kocsis2006bandit} and its variants. These methods are scalable and often yield high-quality policies in practice, but they offer no static correctness guarantees and are therefore unsuitable for safety-critical applications. Moreover, such techniques can be challenging to deploy in practice when they impose unacceptable runtime overhead. 
In contrast, formal methods provide correctness by construction, typically through static (\emph{i.e}., pre-deployment) synthesis techniques and logical specification languages, but they generally do not scale to POMDPs of realistic size or complexity.
In this work, we propose to bridge the gap between sampling-based methods and formal methods by combining them through a learning technique inspired by the $L^*$ algorithm~\cite{Angluin87}. We call our framework \emph{Counterexample-guided Policy Learning Using Sampling} (\cplus). The key insight here is that the synthesis of an observation-based policy can be framed as a learning problem.

In the terminology of $L^*$, \cplus employs membership queries -- called \emph{action queries} -- that are answered by an \emph{action oracle} ($\ao$) 
\del{-- in our case, the \despot{} algorithm, although any off-the-shelf sampling algorithm could be used.} 
that, given a history of observations and actions, returns an optimal action to take from that position. 
Sampling-based POMDP analysis techniques are well-suited to answer these kinds of queries.
The answers to these queries are then generalised to construct the hypotheses -- finite-state controllers (FSCs) represented as Mealy machines. \del{-- through observation tables that are closed and consistent \prm{not very easy to understand at this point} w.r.t. the membership queries \rmv{(answered by \despot{})}{} up to that point, and by generalizing them.}

Once a hypothesis has been constructed, it is passed on to a \emph{model checking oracle} ($\mco$), that verifies whether the candidate controller satisfies the given specification -- these are called \emph{model-checking queries}.
\del{Finally, the algorithm terminates with a Mealy machine that is \emph{equivalent up to model-checking} with the underlying oracle that we use to answer the membership queries.}

If the model checking oracle validates the hypothesis, \cplus terminates with a verified solution, providing the strong guarantees of formal methods. Otherwise, it returns a 
set of traces that violate the specification. These traces are then analyzed to identify a prefix where the candidate controller disagrees with {$\ao$}. 
This prefix, called the \emph{counterexample}, is incorporated (as in $L^*$) into the learning algorithm to refine the hypothesis, and the process is repeated.

While the use of a sampling tool alone provides no static formal guarantees, in our framework it is used exclusively to answer membership queries, with correctness ensured entirely by the subsequent model-checking phase. Importantly, the decision problem for model-checking a candidate policy is tractable, as it reduces to model-checking a Markov chain -- the product of the POMDP and the FSC -- against an infinite-duration safety property. This  can be solved in polynomial time w.r.t. the size{s} of the POMDP and the FSC representing the candidate policy, in contrast to the undecidable nature of the general synthesis problem.

Although due to the undecidability of POMDP synthesis, full completeness is unattainable, we establish a {\em relative completeness} result. First, we prove that finite-state policies always suffice for \emph{threshold-safety objectives} (Theorem~\ref{thm:fsc-existence}). Then, we show that if the policy implicitly defined by $\ao$ is both correct and regular (i.e., recognizable by a finite-state machine), then our algorithm is guaranteed to terminate and produce a correct solution (Theorem~\ref{thm:rel-completeness}). In this case, we obtain two important advantages with respect to executing $\ao$ online. First, the solution produced by our procedure is certified by construction, prior to deployment. Second, the solution takes the form of a finite-state controller, which can be deployed and executed efficiently online, requiring only constant-time updates -- unlike online sampling methods, which may incur latency due to their computational overhead. 

To demonstrate the potential of our approach, we implemented a prototype of \cplus in \textsc{Python}.
For the action oracle $\ao$, we rely on the online planner \despot~\cite{somani2013despot}, which can be queried on demand to compute locally optimal actions without requiring the precomputation of a full policy.
For the model checking oracle $\mco$, our implementation uses \storm~\cite{STORM22}.
The implementation is modular. In particular, other planners -- either sampling-based, or based on formal methods, as discussed later -- can be used as $\ao$. Similarly, other probabilistic model checkers, such as {\sc Prism}~\cite{prism} or {\sc Modest}~\cite{DBLP:conf/tacas/HartmannsH14}, can be used as $\mco$ without modifying the overall framework.

The experimental results are promising. Our implementation, which was developed to assess the potential of the framework, is able to solve threshold-safety problems that are challenging for existing formal-methods algorithms. At the same time, it provides strong formal correctness guarantees through model checking, which sampling-based approaches alone cannot provide. The resulting framework, which combines sampling, learning, and verification, is therefore both theoretically well-founded and practically relevant.

However, due to the undecidability of the POMDP synthesis problem, our approach should be seen as an additional method rather than a complete solution. In particular, it is best used as part of a portfolio of complementary techniques, together with other semi-algorithms based on formal methods. 

\vspace{-0.3cm}
\subsubsection{Related Works.}

The work that is most closely related to our approach is~\cite{WZL21}. This work addresses the learning of nondeterministic FSCs for bounded-horizon reachability objectives in POMDPs, \emph{i.e.}, reaching a set of states within a fixed number of steps with a probability exceeding a given threshold. In contrast to the infinite-horizon properties considered in our approach, bounded-horizon properties are decidable for POMDPs. Moreover, their method does not aim to solve the synthesis problem but to compute a most-permissive policy that can later be refined into a concrete controller. The search for a permissive policy entails an exponential blow-up in the action space, which does not occur in our framework. Their approach is illustrated only on a single, hand-crafted example, and to the best of our knowledge, there is no implementation available for this.

Online algorithms for policy computation include tree search methods that determine policies on the fly. {For instance, the algorithm} POMCP~\cite{silver2010monte} integrates MCTS with particle filtering, while \despot~\cite{somani2013despot} employs deterministic sparse trees to improve scalability.
These approaches scale well but lack strong correctness guarantees, making them less suitable for safety-critical applications.

Classical offline POMDP solvers such as Perseus~\cite{spaan2005perseus} and SARSOP~\cite{kurniawati2008sarsop} use point-based methods~\cite{pineau2003point}, which approximate the value function over a selected subset of the belief space. On the other hand, methods based on partial exploration~\cite{bork2020verification,bork2022under} or abstraction~\cite{norman2017verification} of the underlying belief MDP have been implemented in probabilistic model checkers such as \storm~\cite{STORM22} and Prism~\cite{prism}. 

Inductive synthesis tools such as \paynt~\cite{andriushchenko2021paynt,andriushchenko2023search} explore the space of controllers by incrementally evaluating candidate solutions, but they struggle to scale as the required memory size increases.
{The algorithm described in }\cite{bork2024learning} applies an $L^*$ learning algorithm to synthesize an FSC from a policy precomputed using \storm. Our approach differs from theirs in two key aspects: $(i)$ we do not precompute the full policy but instead query an action oracle only when needed, and $(ii)$ we introduce a novel model-checker-based equivalence query that ensures correctness. Other approaches derive FSCs by extracting automata from {Recurrent Neural Network (RNN)}-based policies for POMDPs~\cite{carr2021task}, but these do not provide formal guarantees.

Black-box checking~\cite{PeledVY02,ShijuboWS21} is also related to our approach, as it combines automata learning with verification in a larger workflow. However, in black-box checking, a positive model-checking result on the learned hypothesis is not sufficient on its own, since one must additionally establish conformance with the underlying system. In contrast, in our setting, once the model checker returns true, the synthesized FSC is already a certified solution to the threshold-safety problem. Moreover, classical black-box checking targets deterministic reactive systems representable as Mealy machines, whereas our framework addresses stochastic, partially observable POMDPs, for which the general synthesis problem is undecidable.

The missing proofs in the article can be found in the Appendix.

\section{Preliminaries}

\label{sec:preliminaries}

A \emph{(discrete) probability distribution} on a countable set $S$ is a function $d: S \to [0,1]$ that satisfies $\sum_{s \in S} d(s) = 1$.
We denote the set of all probability distributions on the set $S$ as $\distribution(S)$. 
For $d \in \distribution(S)$, the \emph{support} of $d$ is $\supp(d) = \{s \in S \mid d(s) > 0\}$. 

\begin{definition}[Markov Chain] A {\em Markov chain (MC)} is a triple $\mc=(\mcstates,\mctrans,\mcinitstates)$ where $\mcstates$ is a countable set of states, $\mctrans : \mcstates \rightarrow \distribution(\mcstates)$ is the transition function and $\mcinitstates\in\mcstates$ is the initial state.
\end{definition} 
For states $s,s'\in S$, $\mctrans(s)(s')$ denotes the probability of moving from state $s$ to state $s'$ in a single step and we denote this probability as $\mctrans(s,s')$.
A \emph{finite path} $\path = s_0s_1\ldots s_i$ of length $i\ge 0$ is
a sequence of $i+1$ consecutive states such that for all 
$t\in[0,i-1]$, $s_{t+1}\in \supp(\mctrans(s_t))$.
Similarly, an infinite path is an infinite sequence $\path = s_0s_1s_2\ldots$ of states such that for all $t\in\N$, $s_{t+1}\in \supp(\mctrans(s_t))$.
The probability of a finite path $\mdppath = s_0s_1 \dots s_n$ in the MC $\mc$, denoted $\mc(\rho)$, is given by $\prod_{i=0}^{n-1} \mctrans(s_i, s_{i+1})$. 
For a finite path $\path$, let $\cylinder(\path)$ denote the set consisting of infinite paths $\path'$
such that $\path$ is a prefix of $\path'$.
{The set} $\cylinder(\path)$ is called the cylinder set of $\path$.
For a set of finite paths $R$, we define $\cylinder(R)=\bigcup_{\path\in R}\cylinder(\path)$.

We denote by $\Pr_{\mc}$ the probability measure over infinite paths induced by the Markov chain $\mc$ (see, e.g.~\cite{baier2008principles}, for a formal definition). Intuitively, $\Pr_{\mc}$ assigns to each measurable set of infinite paths (called \emph{events}) the probability that an execution of $\mc$ will follow a path in that set.

Given two sets of states $\good, \bad \subseteq \mdpstates$, we first define $\Diamond \good$  as the set of paths that eventually reach a state in $\good$: 
$\Diamond \good = \{s_0s_1\ldots\mid \exists i \geq 0\colon s_i\in \good\}.$
This defines a {\em reachability objective}.
We also define $\Box \neg\bad$ as the set of paths that avoid all states in $\bad$:
$\Box \neg \bad = \{s_0s_1\ldots\mid \forall i \geq 0 \colon s_i\not\in \bad\}.$
This defines a {\em safety objective}. 
Finally, we define $\neg\bad\until\good$ as the set of paths that reach a state in $\good$ while avoiding states in $\bad$: $\neg\bad\until\good =
\{s_0s_1\ldots\mid \exists i \geq 0 \colon s_i\in\good \text{~and~} \forall 0 \leq j<i \colon s_j\not\in \bad\}$.
We refer to them as {\em reach-avoid objectives}. 
All these events are measurable. 
In the definitions above, we liberally used the notations of temporal logics ($\Diamond$, $\Box$, and $\until$). An interested reader can see \cite{baier2008principles} for a formal introduction.

\begin{definition}[POMDP]
	A \emph{Partially Observable Markov Decision Process (POMDP)} is a tuple $\pomdp =(\mdpstates, \mdpactions, \mdptransitions, \mdpinitstates, \observations, \observationmap)$ 
	where $\mdpstates$ is a finite set of states, 
	$\mdpactions$ is a finite set of actions, 
	$\mdptransitions: \mdpstates\times \mdpactions\rightharpoonup \distribution(\mdpstates)$ 
    is a (partial) transition function,
	$\mdpinitstates\in\mdpstates$ is the initial state,
	$\observations$ is a finite set of observations, and $\observationmap:\mdpstates\rightarrow\observations$ is an observation function that maps each state to an observation\footnote{Equivalently, we can see an observation $z \in Z$ as the subset of states $\{ s \in \mdpstates \mid \observationmap(s)=z \}$ that share the observation $z$.}. 
\end{definition}

For states $s,s'\in S$, and $a\in \mdpactions$,
$\mdptransitions(s,a)(s')$ denotes the probability of moving from state $s$ to state $s'$ when taking the action $a$ and we denote this probability as $\mdptransitions(s,a,s')$.
If for some state $s\in \mdpstates$ and $a\in \mdpactions$, $\mdptransitions(s,a)$ is not defined, we say that action $a$ is not \emph{enabled} in state $s$.
Note that, when $\observations=\mdpstates$ and $\observationmap$ is the identity, then the POMDP is a classical fully observable MDP. 

In a POMDP $\pomdp$,
an \emph{infinite path} is a sequence $\mdppath = s_0a_0s_1a_1s_2\ldots$ such that for all $t \ge 0$, 
$s_{t+1}\in \supp(\mdptransitions(s_t,a_t))$. 
A \emph{finite path} $\mdppath = s_0a_0s_1\ldots s_i$ of length $i\ge 0$ is
a finite prefix of an infinite path.  
Note that multiple states may have the same observation under $\observationmap$, so the agent cannot in general infer the exact underlying state from its observations.
Thus in a POMDP, the agent observes only \emph{histories} (sequences of observations and actions), not the underlying state sequence.
Formally, a history is a finite sequence of observations and actions
$\history = z_0 a_0 \ldots a_{i-1} z_i$. We call a history \emph{valid}, if  there exists a finite path
$\mdppath = s_0 a_0 s_1 \ldots s_i$ with $z_j = \observationmap(s_j)$ for $0 \leq j \leq i$.
We denote the set of all (valid) histories of $\pomdp$ by $\histories$.  
With a slight abuse of notation, we also use a set of observations $Y\subseteq \observations$ to denote 
the set of states: 
$\{s\in\mcstates\mid \observationmap(s)\in Y\}$. 

\begin{example}
    \label{ex:running-example-pomdp}
Consider the behaviour of a robot in the grid depicted in \Cref{fig:running-example-pomdp}. 
This can be modelled as a POMDP $\pomdp$ as follows:
Each cell in the grid corresponds to a state in $\mdpstates = \{(x,y) \mid x \in \{0,1,2,3\}, y \in \{0,1,2\}\}$. 
The cells marked in red are \emph{bad} states that the robot must avoid.
It also must stay inside the grid (\emph{i.e.}, for example, action $\ua$ from $\bluecell$ leads to a \emph{bad} state).
The robot observes only the cell colour (gray, white, or blue) but cannot distinguish between cells of the same colour.
That is, the set of observations is 
$\observations = \{\redcell, \graycell, \bluecell\}$, 
The set of actions is 
$\mdpactions = \{\ua, \da, \ra\}$. 
The initial state is $\mdpinitstates = (0,1)$. 
The transition function is defined as follows: for each state $(x,y)$ and action $a\in \mdpactions$, the system moves to the intended neighbouring cell with probability $0.9$ and with probability $0.1$ remains at the same cell. 
\end{example}
\begin{figure}[t] 
\begin{minipage}{0.35\linewidth}
\centering
\begin{tikzpicture}[scale=0.7, every node/.style={minimum size=1cm, draw=gray!50, thick, rounded corners}]
\foreach \x/\y in {0/0,1/0,2/0,0/2,1/2,2/2,3/0} {
  \fill[red!70, draw=gray!70, rounded corners] (\x,\y-5) rectangle ++(1,1);
}
\foreach \x/\y in {3/2} {
  \fill[blue!70, draw=gray!70, rounded corners] (\x,\y-5) rectangle ++(1,1);
}
\foreach \x/\y in {0/1,1/1,2/1,3/1} {
  \fill[gray!70, draw=gray!70, rounded corners] (\x,\y-5) rectangle ++(1,1);
}
\foreach \x in {0,1,2,3} {
  \foreach \y in {0,1,2} {
    \draw[white, rounded corners] (\x,-\y-2) rectangle ++(1,-1);
  }
}

%
\draw[line width=2pt, rounded corners] (0,-2) rectangle (4,-5);
%
\foreach \x in {0,1,2,3} {
  \node[above, draw=none] at (\x+0.5,-6.5) {\x};
}
%
\foreach \y in {0,1,2} {
  \node[left, draw=none] at (0,\y-4.5) {\y};
}
\end{tikzpicture}
    \caption{A grid world representing the \Cref{ex:running-example-pomdp}. 
}
    \label{fig:running-example-pomdp}
    \end{minipage}\hfill  
\begin{minipage}{0.6\linewidth}
\centering
    \begin{tikzpicture}[scale=0.6,>=stealth',shorten >=1pt,auto,node distance=2cm,semithick]
  \tikzstyle{every state}=[minimum size=18pt,outer sep=2pt]

  \node[state,initial,initial text={},initial distance=0.7cm] (n0) {$n_0$};
  \node[state]         (n1) [right of=n0] {$n_1$};
  \node[state]         (n2) [right of=n1] {$n_2$};
  \node[state]         (n3) [right of=n2] {$n_3$};

  \path[->]
    (n0) edge[] node[above] {$\graycell / \ra$} (n1)
    (n1) edge[] node[above] {$\graycell / \ra$} (n2)
    (n2) edge[] node[above] {$\graycell / \ra$} (n3) 
    (n3) edge[loop above]   node [align=center]       {$\graycell / \ua$} (n3)
    (n3) edge[loop below]   node [align=center]       {
    $\bluecell / \da$} (n3);
    \fill[gray!0] (1,-3.1) circle (.1pt); 
\end{tikzpicture}
\caption{A finite-state controller for the POMDP in \Cref{ex:running-example-pomdp}.}
\label{fig:fsc-running-example}
\end{minipage}
\vspace{-0.4cm}
\end{figure}

{Policies resolve the nondeterminism in POMDPs by specifying, for each finite history, which action the agent should take next.} 
\begin{definition}[Policy] 
	A (deterministic\footnote{It is well known (see \cite[Lemma~1]{chatterjee2010randomness}) that for infinite duration safety property (and all objectives that can be reduced to it), randomization is not useful to play optimally; therefore, we focus solely on deterministic policies.}) policy for a {\upshape POMDP} $\pomdp$ is a function $\mdpstrategy : \histories\to\mdpactions$. 
\end{definition}

A policy $\sigma$ for a POMDP $\pomdp$ induces an infinite countable MC $\pomdp_{\sigma}$ defined as follows: states are pairs $(s,h)$ where $s\in\mdpstates$ is a state of the POMDP and $h\in\histories$ is a history that ends with the observation $\observationmap(s)$. The initial state is $(\mdpinitstates, \observationmap(\mdpinitstates))$.
The transition function is defined as: 
\small
\[
\mdptransitions^{\sigma}((s,h), (s',h'))
=
\begin{cases}
\mdptransitions(s,\sigma(h),s') & \text{if } h' = h \cdot \sigma(h) \cdot \observationmap(s'),\\
0 & \text{otherwise.}
\end{cases}
\]
\normalsize 
We denote the probability measure associated with this MC as $\prps$. 
We remove $\pomdp$ from the subscript when the POMDP is clear from the context.

The agent controlling the POMDP does not observe the exact state of the system during its execution but only the sequence of observations encountered. The agent maintains a {\em belief}, that is a distribution over the  states of the POMDP, indicating the probability of being in each 
state given the current history.

A particularly relevant class of policies in practice is finite-state policies (also called finite-state controllers), as they naturally give rise to implementations that can be executed efficiently online. 
In \cref{sec:problem-formulation}, we will further establish that such controllers are sufficient to provide solutions for safety threshold specifications.

\begin{definition}[Finite-State Controller]
	A \emph{finite-state controller} {\upshape(FSC)} for a {\upshape POMDP} $\pomdp =(\mdpstates, \mdpactions, \mdptransitions, \mdpinitstates, \observations, \observationmap)$ is a tuple $\fsc = (\fscstates, \fscactionmap, \fsctransitionmap, \fscinitstate)$ where $\fscstates$ is a finite set of nodes, $\fscactionmap: \fscstates \times \observations \rightarrow \mdpactions$ is the action mapping, $\fsctransitionmap:\fscstates \times \observations \rightarrow \fscstates$ is the transition function, and $\fscinitstate$ is the initial node.
\end{definition}

A finite-state controller $\fsc = (\fscstates, \fscactionmap, \fsctransitionmap, \fscinitstate)$ defines a policy $\mdpstrategy_{\fsc}$ in the POMDP $\pomdp =(\mdpstates, \mdpactions, \mdptransitions, \mdpinitstates, \observations, \observationmap)$ as follows: 
it starts from the initial node $n_0$.
At any step, if the current controller node is $n\in \fscstates$, and the current POMDP state is $s\in\mdpstates$ of the POMDP, the action suggested by $\fsc$ is $\fscactionmap(n,\observationmap(s))$. 
The FSC $\fsc$ then updates its current node to $n' = \fsctransitionmap(n,o)$, where $o$ is the newly received observation. The state of the POMDP is updated according to $\mdptransitions$.
Thus, $\pomdp$ and $\fsc$ induce a finite state Markov chain $\pomdp\times\fsc = (\mdpstates \times \fscstates, \mdptransitions^{\fsc}, (s_0,n_0))$ where
\small
\[\mdptransitions^{\fsc}((s,n),(s',n')) =
\begin{cases}
	\mdptransitions(s,\fscactionmap(n,\observationmap(s)),s') & \text{if } n'=\fsctransitionmap(n,\observationmap(s)),\\
	0 & \text{otherwise.}
\end{cases}
\]
\normalsize
We call a policy \emph{finite memory} or \emph{regular},
if it can be represented by an FSC.

\begin{example}
 Consider the POMDP from \Cref{ex:running-example-pomdp}. 
 The objective is to maximize the probability of avoiding the bad states forever.
 A possible policy is to move {\sf right} three times, and then move {\sf up} till the robot reaches the {\sf blue} cell 
 and then move {\sf down} till it reaches the {\sf gray} cell.
 This policy can be represented by an FSC with four nodes as depicted in \cref{fig:fsc-running-example}.
Starting from the cell $(0,1)$, this policy ensures that the robot reaches the cell $(3,1)$ with probability $0.729$ (as there is a probability of $0.9^3$ of moving right three times successfully), and from there it can stay in the safe region forever by moving {\sf up} when at the {\sf gray} cell, and moving {\sf down} from the {\sf blue} cell.  
\end{example}

\section{The Threshold-safety Problem}
\label{sec:problem-formulation}
Given a POMDP, our goal is to identify policies that enforce desirable properties. We first focus on two types of objectives: safety objectives and discounted-sum reward objectives. 

\begin{problem}[Threshold-safety problem]
  \label{prob:safe_synth}
  \vspace{-.4cm}
  \begin{fmpage}{1.028\linewidth} 
  \begin{tabular}{ll}
  {\bf Input}&: A POMDP $\pomdp$, a set of bad observations $\bad$ and a threshold $\threshold \in [0,1)$.\\
  {\bf Output}&: A policy $\sigma$ represented as an FSC s.t. $\prps (\Box \neg {\sf Bad}) > \alpha$, if it exists, \\ &\ \  and $\mathsf{None}$ otherwise. 
\end{tabular} 
  \end{fmpage}
\end{problem}

While it is well-known that the threshold-safety problem is undecidable~\cite{madani1999undecidability},~\Cref{thm:fsc-existence} establishes that \Cref{prob:safe_synth} is well-posed: the existence of a policy enforcing the safety threshold implies the existence of a finite-state {controller} achieving the same threshold. This result justifies our focus on synthesizing FSCs and also ensures that the problem is \emph{recursively enumerable}. 

\begin{restatable}{theorem}{fscExistence} 
\label{thm:fsc-existence}
    Given a {\upshape POMDP} $\pomdp$ with a set of bad observations ${\sf Bad} \subseteq Z$, and a threshold $\alpha \in [0,1)$, if there is a policy $\sigma$ such that $\prps (\Box \neg {\sf Bad}) > \alpha$, then there  exists a finite memory policy $\sigma_{\sf fm}$ such that $\mathbb{P}^{\sigma_{\sf fm}}_{{\cal P}} (\Box \neg {\sf Bad}) > \alpha$. 
\end{restatable}

To prove the statement in \cref{thm:fsc-existence},
we consider the belief MDP $\belief_{\pomdp}$ associated with the POMDP $\pomdp$: its set of states is the infinite set of
beliefs $b \in \distribution(\mdpstates)$, and for each action $a$ and
observation $z$, the successor belief $b'$ is given by the usual
Bayesian update based on $\mdptransitions$ and $\observationmap$, defined as follows: 
  \[
      \forall s \in \mdpstates \cap z: b'(s)=\frac{\sum_{s'\in \mdpstates} b(s') \cdot \mdptransitions(s',a,s)}{\sum_{(s',s'') \mid \observationmap(s'')=z} b(s') \cdot \mdptransitions(s',a,s'')}
  \]
\[
      \forall s \in \mdpstates \cap \overline{z}: b'(s)=0.
      \]

Let $W$ denote the set of belief states from which there exists a policy to ensure safety with probability $1$. Then we can show the following two lemmas (proof in Appendix~\ref{app:problem-formulation}). 

\begin{restatable}{lemma}{finFromW}
 \label{lem:fsc-suffice-safety}
    There exists a finite memory policy $\sigma$ such that for every belief state $b \in W$, $\mathbb{P}^{\sigma}_{\belief_{\pomdp}} (\Box \neg {\sf Bad} \mid \text{initial belief } b) = 1$.
\end{restatable}

\begin{restatable}{lemma}{safeORinW}
    \label{lem:safty-to-reach}
    For any policy $\sigma$ in $\pomdp$, $\mathbb{P}^{\sigma}_{\belief_{\pomdp}} (\Box\neg {\sf Bad} \wedge \Box\neg W)=0$. 
\end{restatable}

We now have all the ingredients to prove Theorem~\ref{thm:fsc-existence}.
\begin{proof}[of \cref{thm:fsc-existence}]
    Assume that there exists a policy $\sigma$ such that $\mathbb{P}^{\sigma}_{\pomdp} (\Box \neg {\sf Bad}) > \alpha$.
    From \cref{lem:safty-to-reach}, 
        $\mathbb{P}^{\sigma}_{\belief_{\pomdp}} (\Diamond W) \geq \mathbb{P}^{\sigma}_{\belief_{\pomdp}} (\Box\neg {\sf Bad} \wedge \Diamond W)
         = \mathbb{P}^{\sigma}_{\belief_{\pomdp}} (\Box\neg {\sf Bad}) - \mathbb{P}^{\sigma}_{\belief_{\pomdp}} (\Box\neg {\sf Bad} \wedge \Box\neg W) 
         = \mathbb{P}^{\sigma}_{\belief_{\pomdp}} (\Box\neg {\sf Bad}) > \alpha$. 
    Note that $\mathbb{P}^{\sigma}_{\belief_{\pomdp}} (\Diamond W) = \lim_{d \to \infty} \mathbb{P}^{\sigma}_{\belief_{\pomdp}} (\Diamond^{\leq d} W)$
    where $\Diamond^{\leq d} W$ denotes the set of paths that reach $W$ within $d$ steps.
    As $\mathbb{P}^{\sigma}_{\belief_{\pomdp}} (\Diamond^{\leq d} W)$ is non-decreasing in $d$, there must exist a finite depth $d$ such that $\mathbb{P}^{\sigma}_{\belief_{\pomdp}} (\Diamond^{\leq d} W) > \alpha$.

    Now we construct $\sigma_{\sf fm}$ as follows: 
    for all histories of depth at most $d$, if no belief state in $W$ has been reached yet, we play according to $\sigma$. 
    Whenever a belief state in $W$ is reached, we switch to the finite-memory policy that enforces safety almost surely from that belief state, which exists by Lemma~\ref{lem:fsc-suffice-safety}. 
    In the remaining case, that is, when we reach a depth greater than $d$ without having reached a belief state in $W$, we play an arbitrary action. 
    Clearly, this policy uses finite memory, as it depends only on the history of observations and actions up to depth $d$.
    And, since from any belief state in $W$, $\sigma_{\sf fm}$ ensures safety almost surely, we have
     \(\mathbb{P}^{\sigma_{\sf fm}}_{\pomdp} (\Box \neg {\sf Bad}) = \mathbb{P}^{\sigma_{\sf fm}}_{\belief_{\pomdp}} (\Diamond W) \geq \mathbb{P}^{\sigma}_{\belief_{\pomdp}} (\Diamond^{\leq d} W) > \alpha\,.
     \)\qed
\end{proof}

\paragraph*{\textbf{Threshold Discounted Reward Problem}.}
In the following, we use an online solver to compute locally optimal actions after a given history of observations and actions. However, existing online solvers are mainly designed to optimize {\em discounted-sum reward objectives}, and do not directly handle infinite-horizon safety properties. We therefore start by formally defining the synthesis problem for discounted-sum objectives. We then show how the threshold-safety synthesis problem can be related to this setting and {\em approximately} reduced to it.

For this, we introduce a few additional notations. A reward function \( r : Z \times A \rightharpoonup \mathbb{Q} \) assigns a reward to each observation-action pair. Let \( \rho = s_0 a_0 s_1 a_1 \dots \) be an infinite path. 
The \emph{total reward} along \( \rho \) is $\sum_{i=0}^{\infty} r({\cal O}(s_i), a_i)$. 
Given a \emph{discount factor} \( \lambda \in (0,1) \), the \emph{discounted reward} along \( \rho \), denoted \( \DRew^{\lambda}(\rho) \), is defined as
\(
\sum_{i=0}^{\infty} \lambda^i \cdot r({\cal O}(s_i), a_i).
\)

We call a POMDP $\pomdp$, when equipped with a reward function $r$, a \emph{reward POMDP}.
For a policy \( \sigma \) in $\pomdp$, we define its expected discounted reward,
denoted by $\mathbb{E}^{\sigma}_{\pomdp} (\DRew^{\lambda})$, 
as the expected value of the discounted rewards over all paths in the Markov chain $\pomdp_\sigma$ that is induced by \( \sigma \). 

\begin{problem}[Threshold discounted reward problem]
  \label{prob:disc_reward}
  \begin{fmpage}{0.95\linewidth}
  \begin{tabular}{ll}
  {\bf Input}&: A POMDP $\pomdp$, reward function $\mdpreward$, discount factor $\discount \in (0,1)$, and \\ &\ \ threshold $\threshold \in \mathbb{Q}$.\\
  {\bf Output}&: A policy $\sigma$ s.t. $\mathbb{E}^{\sigma}_{\pomdp} ({\DRew}^{\lambda}) > \alpha$, if it exists, and  $\mathsf{None}, $ otherwise.  
  \end{tabular}
  \end{fmpage}
\end{problem}

Note that this problem is also undecidable~\cite{madani1999undecidability}. Furthermore, it was shown in~\cite{alur2022framework} that safety (and reachability) objectives cannot be structurally represented as discounted reward objectives in an {\em exact} way. However, we show below that safety objectives can be approximated to any desired degree of accuracy, which is the best achievable outcome in this setting. We also argue that this approximation is sufficient in practice, as supported by our experimental results.

First, for safety objectives, we proceed with the following reduction. 
Given any {set} $\bad \subseteq \observations$, without loss of generality, we will assume that the \emph{unsafe} states (states whose observations are in $\bad$) are sinks.  
Then, from $\pomdp$, we construct a reward POMDP $\pomdpr$ as follows: add one additional state $q$, assign a new observation $\observationmap(q)$, redirect all outgoing transitions from the unsafe states in $\pomdp$ to $q$, and make $q$ a sink {state}. Define the reward function $r$ in $\pomdpr$ as follows: 
for every action $a$,  we assign a reward $-1$ to the observation-action pairs $(o, a)$ for $o\in \bad$, and assign $0$ to every other observation-action pair. It is straightforward to verify that, under this construction, every infinite path in $\pomdpr$ that never visits $\bad$ has a total reward $0$, while every infinite path that visits $\bad$ has a (undiscounted) total reward $-1$. 

In online samplers, the focus is not on the total reward, but rather on the discounted reward. However, when the discount factor $\lambda$ is close to $1$, the difference between the total reward and the discounted reward becomes negligible. Additionally, there is a bijection between paths in $\pomdp$ and $\pomdpr$, which can further be lifted to policies in the two POMDPs. This is because, after visiting a state in $\bad$, all transitions are deterministic in both (and the rest are the same). Combining these two observations, we get that, for all policies $\sigma$:
\begin{equation*}
    \prps(\Box \neg {\sf Bad})=1+\lim_{\lambda \rightarrow 1} \expsp({\sf \DRew^{\lambda}}) 
    = \lim_{\lambda \rightarrow 1}(1 + \expsp({\sf \DRew^{\lambda}}))\,.
\end{equation*}
This result implies that, in the limit, as $\lambda$ tends to $1$, maximizing the expected discounted reward in the $\lambda$-discounted POMDP $\pomdpr$ is equivalent to maximizing the probability of staying safe in $\pomdp$.
We formally state this fact below. 

As recalled above, the discounted-reward problem is undecidable. Nevertheless, sampling-based solvers can compute good, near-optimal policies in many practical settings.

\begin{restatable}{proposition}{safetyDrewEq}
\label{prop:safety-drew-eq}
For all {\upshape POMDPs} $\pomdp$, {and for} all safety objectives encoded by observations in {\sf Bad}, let $\pomdpr$ be the reward {\upshape POMDP} constructed as above. Then, for {any} threshold $\threshold \in [0,1)$ {and $\varepsilon>0$, }:
\begin{enumerate}
    \item[1.]  there exists a discount factor $\lambda \in (0,1)$, such that if there is a policy $\sigma$ satisfying 
$1+\mathbb{E}^{\sigma}_{\pomdpr}({\DRew}^{\lambda}) > \threshold$, then $\mathbb{P}^{\sigma}_{\pomdp}(\Box \neg {\sf Bad}) > \threshold- \varepsilon$;
\item[2.] 
if there exists a policy $\sigma$ such that $\mathbb{P}^{\sigma}_{\pomdp}(\Box \neg {\sf Bad}) > \threshold$, then, there exists a discount factor $\lambda \in (0,1)$ such that $1+\mathbb{E}^{\sigma}_{\pomdpr}({\DRew}^{\lambda}) > \threshold - \varepsilon$.
\end{enumerate}
\end{restatable}

\section{The Framework \cplus}
\label{sec:cplus}
In this section, we present our policy learning framework, named \emph{Counterexample-guided Policy Learning Using Sampling} (\cplus), for the safety-synthesis problem that combines sampling and model-checking through active automata 
learning. 
A diagrammatic description is provided in \Cref{fig:pseudo-algo}. 
For the rest
of this section, we fix a finite POMDP $\pomdp = (\mdpstates, \mdpactions, \mdptransitions, \mdpinitstates, \observations, \observationmap)$ with a set of unsafe observations $\bad$ and a threshold $\threshold \in [0,1)$ inducing a safety specification (as formulated in Problem~\ref{prob:safe_synth}). 
Let $\pomdpr$ be the reward POMDP following the construction described in \cref{sec:problem-formulation}.

Our algorithm is an adaptation of a variant of Angluin's $L^*$ for Mealy 
machines~\cite{shahbaz2009inferring}.
The algorithm maintains a table $\lstartab$ throughout, whose rows and columns are sequences of observations of $\pomdp$. Formally, an observation table is a 
triple $\lstartab=(\S, \E, \T)$, where $\S$ is a nonempty prefix-closed set of observation sequences, $\E$ is a nonempty set of observation sequences, and $\T:(\S \cup \S\cdot\observations)\times  \E \to \mdpactions^*$ is a function that assigns a sequence of actions to observation sequences.
The observation table is initialized with $\S=\{\varepsilon\}$ and $\E = \observations$ (see Table~\ref{tab:running-iter-1}).

\begin{wraptable}{r}{0.4\textwidth}
\centering
\begin{tabular}{p{.5cm}|p{.5cm}|p{.5cm}|p{.5cm}|p{.5cm}|}
    &  & $\redcell$ & $\graycell$ & $\bluecell$ \\
    \hline 
    \hline 
\multirow{1}{*}{\rotatebox{90}{$\mathsf{\S}$}}   
         & $\varepsilon$ & \dct & $\ra$ & \dct \\ 
         \hline
         \hline
\multirow{4}{*}{\rotatebox{90}{\hspace{.25cm}$\mathsf{\S}\cdot \observations$}}   
         & $\redcell$  & \dct & \dct & \dct \\ \cline{2-5} 
         & $\graycell$ & \dct & $\ra$ & \dct \\ \cline{2-5}
         & $\bluecell$ & \dct & \dct & \dct \\ 
\hline
\end{tabular}
\caption{Initial observation table for the POMDP in \Cref{fig:running-example-pomdp}.}
\label{tab:running-iter-1}
\end{wraptable} 
For a row $s = z_1^1 z_1^2  \dots  z_1^m$ and a column 
$e = z_2^1  z_2^2  \dots  z_2^n$, the value of $\T(s, e) $
is a sequence of actions $a_2^1  a_2^2  \dots  a_2^n$.  
Note that,  $|e| = |\T(s, e)|$.  
Further, from the row $s$, the column $e$, and the value $\T(s,e)$, we can obtain an action-observation sequence $z_1^1 a_1^1 z_1^2 a_1^2 \dots z_1^m a_1^m z_2^1 a_2^1 z_2^2 a_2^2 \dots z_2^n a_2^n$, 
where $a_1^i = \T(z_1^1 z_1^2 \dots z_1^{i-1}, z_1^i)$, for $1 \le i \le m$ (with the convention that $z_1^0 = \varepsilon$).
One can observe that, these values already exist in $\lstartab$, since $\S$ is prefix-closed with $z_1^1 z_1^2 \dots z_1^{i-1} \in \S \cup \S\cdot\observations$, and $z_1^i \in \observations \subseteq \E$.  
Lastly, suppose $\E=\{e_1,\ldots,e_{\ell}\}$ for some $\ell\in\N$. Then, for any $s\in\S\cup\S\cdot\observations$, we write $row(s)$ to denote the tuple $(\T(s,e_1),\T(s,e_2), \ldots, (\T(s,e_\ell))$. 

We use two kinds of queries to construct the observation table: \emph{action queries} 
and \emph{model-checking queries}, as described below.  

\begin{itemize}
\item \emph{Action Queries} (\memb): given a sequence of observations and actions $h$, if it is a valid history in $\pomdp$, we query {an action oracle} ($\ao$) to identify the action $a$ which is optimal\footnote{We use a cache to store the action queries instead of querying $\ao$ each time; this avoids selecting different actions for the same history on different occasions.} to take after $h$. 
If $h$ is not a valid history, we return a special symbol `\dc' representing `don't care'. 
The validity of a history can be checked using classical graph theoretic analysis on the POMDP.

\item \emph{Model-Checking Queries} (\equiv): given an FSC $\hypothesis$, we query  
a model-checking oracle ($\mco$)
to verify if $\hypothesis$ satisfies the safety specification, \emph{i.e.}, whether $\Hyposafety$ holds. The oracle $\mco$ either returns \emph{yes} (if the above holds), or a \emph{finite} set of paths in $\pomdp \times \hypothesis$ (called \emph{counterexample}) that are unsafe and have a cumulative probability of at least $1-\threshold$.
\end{itemize}

We now elaborate on different steps of our framework \cplus. 
\begin{figure}[t]
    \centering
    \includegraphics[width=.85\linewidth]{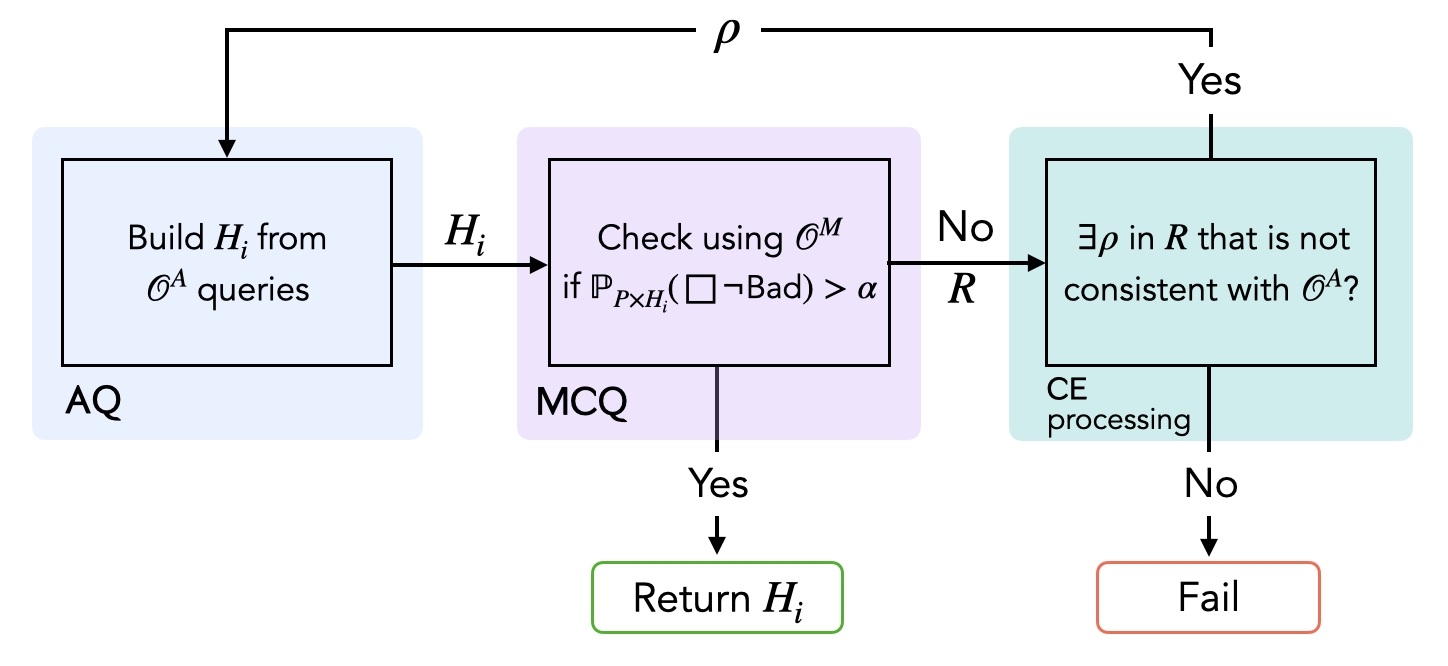}
    \caption{Overview of our policy learning framework \cplus. }
    \label{fig:pseudo-algo}
\end{figure}

\paragraph*{Action Queries via Sampling.}
\label{sec:membership}

Given a \emph{valid} history $h$ of observations and actions, we use a sampling-based $\ao$ to compute an action $a$ to play after $h$ in order to maximize the expected discounted reward of the associated POMDP $\pomdpr$. 
Thanks to~\Cref{prop:safety-drew-eq}, in the following, we will assume that $\ao$ is trying to maximize the probability of satisfying the safety specification in $\pomdp$. 

\begin{remark}
We can also rely on formal methods to construct an action oracle. For instance, we can build a finite abstraction of the belief MDP associated with the POMDP and then compute an optimal memoryless policy that maximizes safety on this abstraction. This can be done, for example, using \storm. The resulting policy can then be queried to answer action requests. In the experimental section, we show that using such an oracle enables \cplus to learn correct policies that are much more compact than the underlying belief MDP, which is a clear advantage, in particular from the perspective of explainability.
\end{remark}

\paragraph*{Hypothesis Construction.}
\label{page:hypothesis-con}
We ensure that the table is \emph{closed}, i.e., for all $s'\in \S\cdot\observations$, there exists $s\in\S$ such that $row(s) = row(s')$\footnote{In contrast to the original $L^*$, we do not need to check for \emph{consistency}, thanks to the counterexample-processing technique of Rivest-Schapire~\cite{rivest1989inference}.}. Once the table is closed, a hypothesis FSC $\hypothesis = (\fscstates, \fscactionmap, \fsctransitionmap, \fscinitstate)$, corresponding to the table, is constructed using standard techniques for $L^*$, as follows. The nodes $\fscstates$ are the rows in $\S$, the initial node $\fscinitstate$ is $row(\varepsilon)$. 
The transition function is defined as $\fsctransitionmap(row(s),o) = row(so)$, and the action map is defined as $\fscactionmap(row(s),o) = \T(s, o)$.

\paragraph*{Model-checking Queries.}
\label{sec:equivalence}
Given a hypothesis $\hypothesis$, we do a model-checking query for $\hypothesis$ using a probabilistic model-checker ($\mco$). We evaluate the quality of $\hypothesis$ by computing the safety probability in the Markov chain $\pomdp \times \hypothesis$. 
If indeed $\Hyposafety$, then the algorithm terminates and returns $\hypothesis$. Else, $\mco$ returns a \emph{finite} counterexample $\ce$.
 \begin{definition}[Counterexample]
 \label{def:counterexample}
For a hypothesis $\hypothesis$ that does not satisfy the safety specification,
a \emph{counterexample} is a set $\ce$ of finite paths in the {\upshape MC} $\pomdp \times \hypothesis$, such that $\cylinder(\ce) \subseteq \Diamond \bad$ and $\Pr_{\pomdp \times \hypothesis}(\cylinder(\ce)) > 1-\threshold$.
 \end{definition}

For safety specifications in finite MC, finite counterexamples suffice~\cite{han2007counterexamples}:

\begin{lemma}[\cite{han2007counterexamples}]
\label{lem:counterexample}
For a finite {\upshape MC} $\mc$, if it does not satisfy a safety specification $\MCsafety$,  then there exists a counterexample ${\cal R}$ that is finite. 
\end{lemma}

When $\mco$ returns a counterexample $\ce$, \cplus proceeds to the next step.

\paragraph*{Counterexample Processing.}
\label{sec:cex-processing}

Let $\ce$ be a finite counterexample returned by $\mco$ in the previous step for hypothesis $\hypothesis$. Assume $\ce$ consists of the finite paths $\rho_1, \rho_2, \dots, \rho_k$. 
Now, all actions taken along these paths are either due to {$\ao$} (\emph{i.e.}, the answers to the action queries), or to the ``generalization'' that occurs during the hypothesis building phase of the learning algorithm.
For all the latter paths, we call {$\ao$} to check whether there exists some $\rho_j$ and a prefix $h$ of $\rho_j$ such that the action after $h$ assigned by $\hypothesis$ and the one returned by {$\ao$} {are different}. 
We then add a suffix of $\rho_j$ as the column of $\lstartab$ using the technique of Rivest-Schapire~\cite{rivest1989inference}. 
In the following lemma, we establish when such a prefix exists. 
We introduce the  notation: given a path 
$\rho := s_0 s_1 \dots$, and a policy $\sigma$, we note $\rho \models \sigma$ when for every $i \ge 0$, $\sigma(z_0a_0z_1\dots z_i) = a_i$, where for every $j \ge 0$, $z_j$ is the observation of the state $s_j$.
The following lemma then 
follows from~\Cref{def:counterexample}. 

\begin{restatable}{lemma}{counterex} 
\label{lem:eq-counterexample}
Let $\hypothesis$ be a hypothesis not satisfying the safety specification, \emph{i.e.}, $\prph(\Box \neg \bad) \leq \alpha$. Further, let $\tsp := \{\sigma \mid \prps(\Box \neg \bad) > \alpha\}$ be the set of all threshold-safe policies.
If $\cex$ is a counterexample for the model-checking query on $\hypothesis$, then $\forall \sigma \in \tsp$, $\exists \rho \in \cex$ such that $\rho \not\models \sigma$. 
\end{restatable}

Due to the undecidability result, there are two possibilities: 
the action oracle $\ao$ may or may not {have} a threshold-safe policy. 
When it {has}, the policy of {$\ao$} belongs to $\mathcal{G}$ of Lemma~\ref{lem:eq-counterexample}. Thus, we are guaranteed to find a prefix $h$ that we add to the observation table $\ObsTable$ 
and \cplus moves to the next iteration.
When {$\ao$} {does not have a threshold-safe} policy, we may not be able to find any prefix $h$ from the counterexample $\cex$ where $\ao$ disagrees with the hypothesis. In that case, \cplus returns \emph{fail}. 

We highlight two key properties of \cplus.
First, when \cplus terminates and returns an FSC, the model-checking step guarantees that the FSC satisfies the safety specification. 
Second,  since the threshold-safety problem is generally undecidable for POMDPs, only relative completeness can be established: 
\cplus may not always terminate, however, whenever {the action oracle} has a regular winning policy (say $\sigma_D$), \cplus terminates. 
This follows from the termination guarantee of $L^*$ and the fact that a counterexample produced in the model-checking step contains a path $\rho \not\models \sigma_D$, due to~\Cref{lem:eq-counterexample}. The following theorem formalizes this idea, and the detailed proof can be found in Appendix~\ref{app:cplus}.

\begin{restatable}{theorem}{relCompl} 
\label{thm:rel-completeness}
Given a {\upshape POMDP} $\pomdp$ with a set of unsafe observations $\bad$ and a threshold $\threshold \in [0,1)$ inducing a safety specification (as formulated in Problem~\ref{prob:safe_synth}),
 let $\pomdpr$ be the corresponding reward {\upshape POMDP}. 
\begin{enumerate}
    \item[$\bullet$ (Correctness)]
If  {\upshape \cplus} terminates and returns an {\upshape FSC} $\hypothesis$, then it  satisfies the specification, \emph{i.e.}, $\Hyposafety$. 
    \item[$\bullet$ (Relative completeness)]
For any $\varepsilon > 0$, if {the action oracle} (with the corresponding $\lambda$ as in~\cref{prop:safety-drew-eq}) follows a regular policy $\sigma$, such that it satisfies $1+\mathbb{E}^{\sigma}_{\pomdpr}({\DRew}^{\lambda}) > \alpha+\varepsilon$ in $\pomdpr$, then {\upshape\cplus} terminates and returns an {\upshape FSC} (guaranteed to be correct by the correctness property above). 
\end{enumerate} 
\end{restatable}

\subsubsection*{Illustrative example.} We now demonstrate the execution of \cplus on the POMDP $\pomdp$ in \Cref{fig:running-example-pomdp}.
We first initialise $\S$ to $\{\varepsilon\}$ and $\E$ to $\{\redcell,\graycell, \bluecell\}$. The initial observation table built using $\ao$ is shown in \Cref{tab:running-iter-1}. 
Since the robot starts in the cell (0,1), the first observation seen can only be $\graycell$. 
Therefore, the table contains \dc\ in the cells corresponding to the other two observations. 
    This observation table is closed. 
    Therefore, we construct the hypothesis FSC $\hypothesis$ in \Cref{fig:init-fsc-running-example} from this observation table as described in Page~\pageref{page:hypothesis-con}. 
    In this example, we query $\mco$ to model-check the threshold-safety formula: 
    $\Pr_{\pomdp\times\hypothesis}(\Box\neg\bad)>0.7$. 
    In this case, $\mco$ returns a counterexample $\rho := \graycell \graycell \graycell \graycell$. 
    Note now that the observation-action sequence $\graycell \ra \graycell \ra \graycell \ra \graycell \ra$ results in a bad observation (i.e., it goes outside the grid) with probability $0.9^4 = 0.65 > 0.3$.
    From the observation sequence $\rho$, we obtain the suffix $\graycell\graycell\graycell$ using Rivest-Schapire counterexample processing and add this to $\E$. We fill the resultant observation table using queries to $\ao$. The resultant table is shown in \Cref{tab:running-iter-2}. Since this table is not closed (the row labelled with $\graycell$ is not in $\S$), we make it closed. This gives us the observation table shown in \Cref{tab:running-iter-3}. 
    From this, we construct the FSC given in \Cref{fig:fsc-running-example} and query $\mco$ again. The query now succeeds, and \cplus returns this FSC.

\begin{figure}[t]
\centering

\begin{minipage}{0.19\linewidth}
    \centering
    \begin{tikzpicture}[scale=0.6,>=stealth',shorten >=1pt,auto,node distance=2.5cm,semithick]
        \tikzstyle{every state}=[minimum size=18pt,outer sep=2pt]

        \node[state,initial,initial text={},initial distance=0.7cm] (n0) {$n_0$};

        \path[->]
            (n0) edge[loop above] node [align=center] {$\graycell / \ra$} (n0);
             \fill[gray!30] (0,-1.1) circle (.1pt);
    \end{tikzpicture}
    \caption{\centering Hypothesis after the first iteration.} 
    \label{fig:init-fsc-running-example}
\end{minipage}
\begin{minipage}{0.39\linewidth}
    \centering
    \scalebox{0.8}{
    \begin{tabular}{p{.5cm}|p{.5cm}|p{.5cm}|p{.5cm}|p{.5cm}|p{1.5cm}|}
        &  & $\redcell$ & $\graycell$ & $\bluecell$ & $\graycell\ \graycell\ \graycell$ \\
        \hline 
        \hline 
        \multirow{1}{*}{\rotatebox{90}{$\mathsf{\S}$}}   
             & $\varepsilon$ & \dct & $\ra$ & \dct & $\ra \ra \ra$ \\ 
             \hline
             \hline
        \multirow{4}{*}{\rotatebox{90}{\hspace{.2cm}$\mathsf{\S}\cdot \observations$}}   
             & $\redcell$  & \dct & \dct & \dct & \dct \dct \dct \\ \cline{2-6} 
             & $\graycell$ & \dct & $\ra$ & \dct & $\ra \ra \ua$ \\ \cline{2-6}
             & $\bluecell$ & \dct & \dct & \dct & \dct \dct \dct \\ 
        \hline
    \end{tabular}
}
    \captionof{table}{{\centering Observation table after counterexample processing.}}
    \label{tab:running-iter-2}
    
\end{minipage}
\begin{minipage}{0.39\linewidth}
\scalebox{0.8}{
    \begin{tabular}{p{.5cm}|p{1.8cm}|p{.5cm}|p{.5cm}|p{.5cm}|p{1.5cm}|}
    &  & $\redcell$ & $\graycell$ & $\bluecell$ & $\graycell\ \graycell\ \graycell$ \\
    \hline 
    \hline 
\multirow{1}{*}{\rotatebox{90}{\hspace{-.7cm}$\mathsf{\S}$}}   
         & $\varepsilon$ & \dct & $\ra$ & \dct & $\ra \ra \ra$ \\ \cline{2-6}
         
         & $\graycell$ & \dct & $\ra$ & \dct & $\ra \ra \ua$ \\ \cline{2-6}

         & $\graycell \graycell$ & \dct & $\ra$ & \dct & $\ra \ua \ua$ \\ \cline{2-6}
         & $\graycell \graycell \graycell$ & \dct & $\ua$ & \dct & $\ua \ua \ua$ \\ \cline{2-6}
         & $\graycell \graycell \graycell \graycell$ & \dct & $\ua$ & $\da$ & $\ua \ua \ua$ \\ \cline{2-6}
         \hline
         \hline
\multirow{4}{*}{\rotatebox{90}{\hspace{.9cm}$\mathsf{\S}\cdot \observations$}} 
         & $\graycell \graycell \graycell \graycell \graycell$ & \dct & $\ua$ & $\da$ & $\ua \ua \ua$ \\ \cline{2-6}
         & $\graycell \graycell \graycell \graycell \bluecell$ & \dct & $\ua$ & $\da$ & $\ua \ua \ua$ \\ \cline{2-6}
\hline
\end{tabular}
}
\captionof{table}{\centering Observation table after making  \Cref{tab:running-iter-2} \emph{closed}. Rows with only $\dc$ are omitted.  
}
\label{tab:running-iter-3}
\end{minipage}
\vspace{-0.2cm}
\end{figure}

\begin{remark}
Compared to using a sampling-based oracle directly online, our method offers two main advantages. First, it is static: the controller is synthesized prior to deployment. This avoids the risk of execution-time failures that may occur when relying on an online planner, for instance, when no solution exists. Second, if the algorithm terminates, it produces a formally verified finite-state controller. This controller has been validated through a model-checking step and can be executed with minimal overhead. In particular, it requires only constant-time updates of its internal state, which makes it suitable for real-time or resource-constrained environments. In contrast, even when the sampling oracle is able to synthesize a correct controller online, its runtime computational cost may make such solutions impractical, especially in settings with strict latency requirements.
\end{remark}

Finally, note that the FSC, if any, returned by our framework need not be the same as the one implicitly defined by {the action oracle}  because the learning procedure may terminate before fully constructing the policy {of the action oracle}. Nonetheless, due to the model-checking step, the returned controller will also satisfy the threshold-safety specification~(see \Cref{thm:rel-completeness}).
\vspace{-0.1cm}
\subsubsection*{Beyond Safety Objectives.}
So far, we have focused on safety objectives. This choice was mainly motivated by the simplicity of counterexamples, which can be represented as finite sets of finite paths violating the safety threshold, allowing efficient integration of model checking into the learning loop. Other objectives, such as infinite-horizon reachability, are more challenging, as counterexamples may involve infinite paths or more complex structures, for instance, subgraphs of the induced Markov chain that violate the reachability constraint.

To handle reachability, we rely on a standard reduction from unbounded to bounded reachability, where the target observation must be reached within a fixed number of steps. This reduction turns reachability into a safety problem and provides an approximation that converges to the unbounded objective as the bound increases. We use this approach to validate reachability synthesis tasks. Importantly, if the model-checking phase in \cplus succeeds for a bounded reachability objective, then the resulting FSC also satisfies the corresponding unbounded reachability property.

Finally, our framework naturally supports \emph{reach-avoid} objectives, which combine reaching a target set of observations with avoiding unsafe observations, and are both expressive and practically relevant.

\section{Experiments}  
\label{sec:experiments}

\subsubsection*{Implementation Details.}
We have made a prototype implementation of $\cplus$\footnote{publicly available at \url{https://github.com/mukherjee-sayan/pomdp-policy-synth}} in {\sc Python} using the automata learning library \aalpy~\cite[v1.4.1]{muvskardin2022aalpy}. 
The POMDPs were modelled in PRISM format and then translated to the DRN format using the model checker \storm~\cite{STORM22}.
We built a simulator for DRN files to handle the action queries, in the online planner \despot~\cite{somani2013despot}.
\despot employs an anytime heuristic search algorithm to incrementally construct a finite-depth sparse approximation of the belief tree, maintaining and updating lower and upper bounds on the value of each belief node. 
These bounds guide the search and expansion process by focusing on the most promising parts of the tree. The algorithm finally selects and returns the best action at the root belief node based on the current value estimates.
To optimize the amount of computation performed by \despot, and also to get rid of possible non-determinism in the action queries due to sampling, 
for every sequence of observations that we query on \despot, we store the actions suggested by \despot in a cache. This means, given two sequences $h,h'$ with $h$ being a prefix of $h'$, when computing the actions for $h'$, we call \despot only on the suffix of $h'$ which is not present in $h$.
If given sufficient time, \despot returns an optimal action with high probability (see \cite{somani2013despot} for a formal description). 
We use the default settings of \despot, in particular, the discount factor $\lambda$ is set to $0.95$ \footnote{The fixed $\lambda$ value was chosen as a proof-of-concept. We acknowledge that the resulting $\varepsilon$ (in \cref{prop:safety-drew-eq}) is not quantified per example, and that longer-horizon instances may be more sensitive to this choice. }.

Model-checking queries were implemented in \storm~\cite{STORM22}.
When a counterexample set exists for the given safety threshold $\threshold$ and a hypothesis FSC, \storm finds one by 
computing the shortest paths until the accumulated probability mass of the traces exceeds $\threshold$. 
It uses a recursive enumeration algorithm~\cite{jimenez1999computing} for finding the shortest paths in a weighted graph 
(see~\cite{han2007counterexamples}).

\subsubsection*{Experimental Details.}
In our implementation, we instantiate \cplus with \despot as action oracle and \storm as the model-checking oracle (denoted as $\cplusds$). 
In the benchmarks where $\storm$ performs well, we also use \storm as both the action and model-checking oracle (denoted as $\cplusss$). 
When using \despot as $\ao$, \cplus queries it on demand as needed.
When using \storm, we first compute a policy offline on an approximation of the infinite beliefs MDP, and \cplus then uses this precomputed policy to answer action queries.
In this setting, \storm explores the explicit model up to its internal cutoffs, which may still involve analyzing parts of the belief space that \cplus never queries.  
Also, \cplus may query beliefs that lie outside \storm's explored region, in which case we fall back to \despot for the action queries. 
All experiments were conducted on an Intel{\textsuperscript{\textregistered}} Xeon{\textsuperscript{\textregistered}} Gold 5218R CPU 
running   Ubuntu 22.04.5 LTS, with 20 CPU cores and 251~GB RAM. 

We compare our results with two state-of-the-art approaches based on formal methods: 
$(i)$ the \emph{belief exploration} approach of \storm~\cite{STORM22} constructs a finite fragment of the belief MDP by exploring successor beliefs until a limit is reached, then approximates the value of the \emph{cutoff} beliefs (the beliefs where the exploration stopped) using policies derived from the underlying MDP, 
and
$(ii)$ the \emph{iterative policy synthesis} approach of \paynt~\cite{andriushchenko2021paynt} that inductively explores and evaluates policies of increasing memory size. 
While \storm is a mature probabilistic model checker, 
its belief-exploration approach often struggles with \emph{safety} objectives, as it does not guarantee a safe policy from the \emph{cutoff} beliefs. 
\paynt, on the other hand, synthesizes compact FSCs; however, it may take significantly longer on instances where small FSCs\footnote{{\sc Paynt} usually does not terminate when solutions require FSCs beyond a few nodes.} are not sufficient.  

In \storm, we allocate $60$ seconds for the belief exploration phase to construct a substantial belief fragment and report the value achieved.
We set a timeout of $600$ seconds for \storm, \paynt, and also for \cplus, with each of \despot and \storm as $\ao$. 
For \paynt and \cplus, we target the same safety thresholds.
Since we cannot run \storm with different thresholds, we report on the single optimal value computed by \storm for all the benchmarks.
For \paynt, we extract from the solver logs the time at which an FSC satisfying each target safety threshold is first obtained, or report failure if no such FSC is found within the time limit. For \cplus, synthesis is performed independently for each threshold provided as input, and we report the corresponding synthesis times. 

The main objective of this section is to demonstrate through empirical results that even with a prototypical implementation of the \cplus framework, we can handle models that are challenging for existing formal-methods-based approaches. 
The key takeaways from the empirical results are the following: 
    $(i)$ \storm often fails to produce useful policies for infinite horizon \emph{safety} objectives, since its Belief exploration fails to guarantee safety from cutoff belief states (Table~\ref{tab:grid-world}), 
    $(ii)$ \paynt struggles as soon as it needs to compute FSCs even  with a few number of nodes, since  iterative policy synthesis underperforms due to the rapid exponential growth of the solution space when the number of nodes in the candidate FSCs increases (Table~\ref{tab:other-benchmarks}), 
    $(iii)$ \cplus is less sensitive to the size of the learnt FSC (see the examples \emph{hallway}). However, it struggles when the number of observations in the model grows large (this trend is visible in Table~\ref{tab:other-benchmarks}), due to the inherent weakness of $L^*$ on large alphabets, 
    $(iv)$ For the examples where \storm performs well, using \storm as $\ao$ improves the performance of \cplus (see the \emph{cards} examples). 
These results suggest that no single approach dominates across all instances, underscoring the value of a portfolio approach to policy synthesis for POMDPs (which is undecidable).

\subsubsection*{Evaluation.}
We evaluate \cplus on different benchmarks that demonstrate the framework's capability, and present our results in 
 \Cref{tab:grid-world,tab:other-benchmarks}. 
Here, $|\hypothesis|$ reports on the number of states present in the constructed hypothesis; $|\belief|$ is the number of states in the belief MDP computed by \storm; `iter' refers to the number of iterations in \cplus; `time' reports the total time taken by the tools in seconds; `value'  is the maximum probability to satisfy the specification with the policy computed by \storm.
\begin{table}[t]
\centering
\small
\begin{tabular}{l|c!{\vrule width 1.2pt}c|c|c|c!{\vrule width 1.2pt}c|c|c!{\vrule width 1.2pt}c|c|c}
\hline
\multirow{2}{*}{Grid Size} & \multirow{2}{*}{$\alpha$}
& \multicolumn{4}{c!{\vrule width 1.2pt}}{$\cplusds$}  
& \multicolumn{3}{c!{\vrule width 1.2pt}}{$\paynt$} 
& \multicolumn{3}{c}{$\storm$} \\ [2pt]
\cmidrule{3-12} 
 &
 & success
 & $|\hypothesis|$ 
 & iter 
 & time
 & success
 & $|\hypothesis|$ 
 & time
 & $|\belief|$ 
 & value
 & time\\[3pt] 
\midrule

\multirow{2}{*}{$5\times5$} 
        & 0.2 & 5/5 & 3.2 & 1 & 4.81
       & 5/5 & 1 & 0.004
       & \multirow{2}{*}{$17.8\times 10^6$} & \multirow{2}{*}{0} & \multirow{2}{*}{227.12}\\ 

        & 0.5 & 5/5 & 3.2 & 1 & 4.81
       & 5/5 & 1 & 0.004
       & & &\\
\midrule
\multirow{2}{*}{$10\times 10$} 
        & 0.2 & 5/5 & 2.8 & 1 & 4.58 
       & 5/5 & 1 & 0.002 
       & \multirow{2}{*}{$10.5\times 10^6$} & \multirow{2}{*}{0} & \multirow{2}{*}{306}\\ 

        & 0.5 & 2/5 & 2 & 1 & 3.36 
       & 5/5 & 1 & 0.002
       & & & \\
\midrule
\multirow{2}{*}{$15\times 15$} 
        & 0.2 & 5/5 & 3 & 1 & 5.53 
       & 5/5 & 1 & 0.006
       & \multirow{2}{*}{$5.4\times 10^6$} & \multirow{2}{*}{0} & \multirow{2}{*}{280.43}\\ 

        & 0.5 & 1/5 & 3 & 1 & 5.54 
       & 5/5 & 1 & 0.006
       & & &\\
\midrule
\multirow{2}{*}{$20\times 20$} 
         & 0.2 & 5/5 & 3.4 & 1 & 7.20 
       & 5/5 & 1 & 0.01
       & \multirow{2}{*}{$3.2\times 10^6$} & \multirow{2}{*}{0} & \multirow{2}{*}{336.68}\\ 

        & 0.5 & 0/5 & -- & -- & -- 
       & 4/5 & 1 & 0.01
       & & &\\
\midrule

\end{tabular}
\caption{Evaluation on \emph{Grid-world} benchmarks. The columns `success' denote the number of times the tools could construct a policy satisfying threshold ($\alpha$).} 
\label{tab:grid-world}
\vspace{-0.5cm}
\end{table}

\paragraph{Grid-World Benchmarks} 
is a classical domain in POMDP research: a robot moves in an N×N grid, aiming to avoid holes while navigating in cardinal directions. Movements are stochastic -- with $90\%$ probability the move succeeds, and with $10\%$ the robot remains in place. The robot starts at a random safe cell, and it cannot distinguish between safe cells, leading to partial observability. The goal for the robot is to maximize the probability of staying safe, that is, avoiding visiting the holes (bad cells) in specific locations. 
We generate grids of size $N\times N$ for $N\in\{5, 10, 15, 20\}$, 
with $5$ random instances per size. 
$10\%$ of the cells in the grid were chosen at random and were assigned as bad cells. 
The average runtime, success rate, and synthesized controller size for different $\threshold$ and grid sizes are reported in \cref{tab:grid-world}. 
In these randomly generated instances, 
even FSCs of size $1$ typically suffice, 
and thus \paynt performs well, synthesizing controllers quickly. 
But because of the stochastic nature, the belief space for these POMDPs is unbounded, even for small grids with 25 cells. 
Thus, in these safety examples, \storm's belief-exploration approach fails: even with an additional $60$ seconds exploration budget, \storm fails to find good strategies in any of the instances, even when it explores the belief space up to size approximately $10^7$.

\begin{table}[t]
\centering
\resizebox{1\linewidth}{!}{%
\begin{tabular}{l|c|c|c!{\vrule width 1.2pt}c|c|c!{\vrule width 1.2pt}c|c|c!{\vrule width 1.2pt}c|c!{\vrule width 1.2pt}c|c|c}
\toprule
\multirow{2}{*}{Model} 
& \multirow{2}{*}{$|\mdpstates|$} 
& \multirow{2}{*}{$|\observations|$} 
& \multirow{2}{*}{$\alpha$}
& \multicolumn{3}{c!{\vrule width 1.2pt}}{$\cplusds$} 
& \multicolumn{3}{c!{\vrule width 1.2pt}}{$\cplusss$} 
& \multicolumn{2}{c!{\vrule width 1.2pt}}{$\paynt$}
& \multicolumn{3}{c}{$\storm$} \\ [2pt]
\cline{5-15}

& & & 
& $|\hypothesis|$ & iter & time
& $|\hypothesis|$ & iter & time 
& $|\hypothesis|$ & time
& $|\belief|$ & value & time\\
\midrule
\multirow{2}{*}{} & \multirow{2}{*}{} & \multirow{2}{*}{} &  
& & &  
& & &
& & & 
& \multirow{2}{*}{} & \multirow{2}{*}{} \\[-6pt]
hallway & 16 & 3 & 0.99 & 25 & 3 & 78.11 & 24 & 3 & 49.32 & 5 & \textbf{1.44} & 3 & 1.00 & 0.01\\
hallway-2 & 24 & 3 & 0.99 & 20 & 4 & \textbf{68.99} & 20 & 3 & 43.41 & - & - & 3 & 1.00 & 0.01\\
hallway-simple-50 & 154 & 4 & 0.99 & 53 & 2 & 138.77 & 53 & 2 & \textbf{22.61} & - & - & 156 & 1.00 & 1.00\\
\midrule
cards-removed-2 & 9 & 6 & 0.5 & 4 & 1 & 1.59 & 4 & 1 & 1.48 & 1 & \textbf{0.01} & 10 & 1.00 & 0.01\\
cards-removed-3 & 76 & 10 & 0.5 & 7 & 1 & 2.10 & 20 & 2 & 4.04 & 1 & \textbf{0.02} & 103 & 0.97 & 1.00\\
cards-removed-4 & 165 & 12 & 0.5 & 48 & 2 & 12.28 & 29 & 2 & 7.42 & 2 & \textbf{0.16} & 322 & 0.91 & 2.00\\
cards-removed-5 & 306 & 14 & 0.5 & 118 & 2 & 44.74 & 47 & 2 & \textbf{16.13} & - & - & 919 & 0.85 & 8.00\\
\midrule
cards-added-2 & 35 & 8 & 0.5 & 7 & 1 & 1.93 & 7 & 1 & 1.77 & 1 & \textbf{0.01} & 40 & 0.74 & 0.01\\
cards-added-3 & 94 & 10 & 0.5 & 29 & 4 & 11.80 & 21 & 5 & 10.66 & 2 & \textbf{0.03} & 320 & 0.62 & 1.00\\
cards-added-4 & 197 & 12 & 0.5 & 356 & 9 & 224.87 & 108 & 19 & \textbf{184.11} & - & - & 2526 & 0.56 & 0.01\\ 
cards-added-5 & 356 & 14 & 0.5 & - & - & - & - & - & - & - & - & 19603 & 0.51 & 93.00\\ 
\midrule
cheese-maze & 15 & 8 & 0.8 & 14 & 1 & 7.95 & 12 & 1 & 2.83 & 1 & \textbf{0.01} & 1.7$\times 10^7$ & 0.99 & 231.64 \\

cheese-maze-det & 15 & 8 & 0.8 & 9 & 1 & 6.10 & 9 & 1 & 1.52 & 2 & \textbf{0.01} & 22 & 1.00 & 0.01\\

refuel (N=6,E=8) & 270 & 36 & 0.5 & 3 & 1 & 4.97 & 3 & 1 & 1.46 & 1 & \textbf{0.07} & 2.6$\times 10^7$ & 0.72 & 288.55\\

rocks(N=4)  & 332 & 66 & 0.5 & 12 & 2 & 32.61 & 26 & 1 & 15.59 & 1 & \textbf{0.05} & 485 & 1.00 & 7.00\\
 \bottomrule
\end{tabular}
}
\caption{Evaluation on various benchmarks for the tools \cplusds, \cplusss, \storm, and \paynt. We report only the bounded version of the cards here. Bold entries indicate the best-performing tool (in terms of time) among those that produce verified FSCs. Extended experimental results are included in the Appendix (\Cref{tab:other-benchmarks-full}).}
\vspace{-0.6cm}
\label{tab:other-benchmarks}
\end{table}

\paragraph{Hallway Navigation Benchmarks} 
is a variant of the examples in~\cite{kaelbling1995learning}, where a robot is inside a hallway, and its objective is to maximize the probability of \emph{reaching a specific location}.
The robot can move in different directions, but its movement is restricted by the walls in the hallway.
Similar to the Grid-World scenario, the robot starts from a random position in the hallway and cannot distinguish among the cells. To further show the advantage of our approach, we  also  consider another variant of the hallway problem: the robot must go \emph{right} $50$ times and then go \emph{down} to reach the target (\emph{hallway-simple-$50$}).
The objective in all these benchmarks is to compute a policy for the threshold $0.99$.

On these benchmarks, \storm returns randomized policies (within $<1$ second) that reach the target with probability $1$. \paynt succeeds in finding a policy in the smaller variant, but times out in the longer variants. 
Since \cplus cannot handle unbounded reachability, we reformulated this objective as bounded reachability with a horizon of $100$. 
We verified that the policy computed for bounded reachability is also sufficient for the unbounded reachability setup. 
On the \emph{hallway-simple-50} benchmark, while the winning policy is conceptually simple, the size of the required controller is rather large (to track the number of steps taken). 
On this example, \cplus could find the optimal policy in just $2$ iterations within 3 minutes, whereas  \paynt{}'s inductive approach struggles to produce one, illustrating that \cplus's counterexample-guided learning scales better than exhaustive search when large FSCs are necessary.

\paragraph{Cards Game Benchmarks}
is a more intricate example, adapted from~\cite{chatterjee2025value} (see \cref{app:cards} for a version of the example).
The game involves an $N$-deck card.
In the \emph{removed} variant, one card is initially removed at random; in the \emph{added} variant, one card is duplicated. 
At each step, the player draws a card from the deck uniformly at random or guesses which card is missing or duplicated. 
The objective is to guess correctly (reachability) while avoiding incorrect guesses (safety). 
We consider both variants under two settings: 
(i) \emph{bounded} -- the agent must guess the card within $2N$ steps; and
(ii) \emph{unbounded} -- there is no fixed window to guess the correct card, but at each step, with probability $0.05$, the agent is forced to guess the correct card.
These represent challenging benchmarks requiring substantial controller memory. 
In the \emph{removed} variant, the optimal policy is to wait until observing all but one card, then guess the unseen card. 
In the \emph{added} variant, the policy is to track which card appears most frequently in draws, then guess that card as the duplicate.

\cref{tab:other-benchmarks} reports results across both variants for bounded settings for different values of $N$. 
\paynt performs well on smaller examples, where the FSC size remains manageable, but its performance degrades rapidly as $N$ increases due to the exponential growth of the controller space. 
In contrast, \cplus scales effectively: it synthesizes better FSCs, which avoids exhaustive enumeration of the controller space. 
\storm's belief-exploration approach works well on these benchmarks, quickly computing high-quality policies. 
When we use \storm's computed policy as the action oracle for \cplus (instead of \despot), synthesis times improve significantly, 
demonstrating the flexibility of \cplus's framework to incorporate different oracles.

\paragraph{Reach-Avoid Navigation Benchmarks.} 
We evaluate \cplus on several bigger reach-avoid navigation benchmarks inspired by \cite{kaelbling1995learning,smith2012heuristic,mccallum1992first,junges2021enforcing}. 
The benchmark \emph{Refuel} involves a rover travelling from one corner to the opposite of a grid, while avoiding obstacles on the diagonal; it must manage energy by recharging at designated stations and operate under noisy observations of position and battery level. 
\emph{Rocks} requires the rover to sample two rocks (valuable or dangerous), collect a valuable rock, deliver it to the drop-off zone, and avoid dangerous rocks.
In \emph{Cheese-Maze}, a mouse must collect cheese while avoiding traps. It has two variants: deterministic variant (actions succeed as intended) and stochastic variant (intended action succeeds with probability $0.8$).
In these reach-avoid examples, both \storm and \paynt work well. Furthermore, when we use \storm's computed policy as the action oracle, our synthesis time improves significantly.

\subsubsection*{Data Availability.} The artifact accompanying this paper \cite{artifact} contains source code, benchmark files, and scripts to replicate our experiments.

\section{Conclusion and Future Work}
\label{sec:conclusion}
In this work, we propose \cplus{} -- a framework that combines the scalability of sampling-based POMDP controller synthesis methods with the formal guarantees of model checking. Our prototype demonstrates that \cplus{} can synthesize verified controllers across several benchmarks.
Theoretically, we prove that if the underlying sampler (in our case, \despot{}) induces a winning policy that is regular, then \cplus{} is guaranteed to terminate with a correct FSC. 

We show that this new approach for solving infinite-duration safety objectives in POMDPs often compares favourably with state-of-the-art tools such as \storm{} and \paynt{}. Given that POMDPs constitute a particularly challenging class of models, \cplus{} can be regarded as a complementary framework that can be used within a portfolio of methods alongside existing approaches proposed in the literature. We will investigate whether this algorithm can be integrated into state-of-the-art frameworks such as \storm{}~\cite{STORM22} or {\sc Prism}~\cite{prism}.

For future work, we envision several directions.
First, to address POMDPs with large observation spaces, we plan to explore other techniques such as symbolic automata learning that can handle large (or even infinite) alphabets~\cite{learn-symb-aut}.
Recent works on compositional automata learning~\cite{neele2023compositional,labbaf2023compositional,fujinami2025componentwise} suggests another promising direction: instead of learning a single FSC monolithically, one could exploit structural decomposition to improve scalability. 
Second, \cplus{} provides a flexible foundation: each of its core 
components (the oracles) 
can be replaced with alternative tools, enabling broader applicability across diverse models and objectives.
Third, we aim to generalize the use of counterexamples to handle (general) reachability and, more generally, PCTL properties 
within our framework.
Finally, the idea of combining heuristic search (here, sampling) with decidable model checking and learning can also be applied in other settings where the synthesis problem is undecidable. For example, this framework could be applied to the synthesis problem for weighted timed automata. As in the case of POMDPs, this synthesis problem is undecidable~\cite{DBLP:conf/formats/BrihayeBR05}, while model checking remains decidable~\cite{DBLP:journals/fmsd/BouyerBBR07}. 

\section*{Acknowledgements}

\emph{Debraj Chakraborty} is currently supported by the National Research Foundation, Singapore, under its RSS Scheme (NRF-RSS2022-009) and was also partially supported by the MASH (MUNI/I/1757/2021) funded by Masaryk University and the ERC project InOVationCS (grant No. 101171844). 
\emph{Anirban Majumdar} is supported by the Department of Atomic Energy, Government of
India, under project no. RTI4014.
\emph{Prince Mathew} is supported by the FNRS--DFG Weave project FORM-LEARN-POMDP (Ref.\ 40028647). 
Most of this work was done while \emph{Sayan Mukherjee} was affiliated with Univ Rennes, Inria, CNRS, IRISA, France; he is currently supported by IITB Trust Lab. 
\emph{Jean-Fran\c{c}ois Raskin} is supported by the Foundation ULB and the Fonds Thelam of Roi Baudouin Foundation.

\bibliographystyle{splncs04}
\bibliography{ref}

@inproceedings{DBLP:conf/tacas/HartmannsH14,
  author       = {Arnd Hartmanns and
                  Holger Hermanns},
  title        = {The Modest Toolset: An Integrated Environment for Quantitative Modelling
                  and Verification},
  booktitle    = {{TACAS}},
  series       = {Lecture Notes in Computer Science},
  volume       = {8413},
  pages        = {593--598},
  publisher    = {Springer},
  year         = {2014}
}

@inproceedings{han2007counterexamples,
  title={Counterexamples in probabilistic model checking},
  author={Han, Tingting and Katoen, Joost-Pieter},
  booktitle={International Conference on Tools and Algorithms for the Construction and Analysis of Systems},
  pages={72--86},
  year={2007},
  organization={Springer}
}

@article{chatterjee2007algorithms,
	title={Algorithms for omega-regular games with imperfect information},
	author={Chatterjee, Krishnendu and Doyen, Laurent and Henzinger, Thomas A and Raskin, Jean-Fran{\c{c}}ois},
	journal={Logical Methods in Computer Science},
	volume={3},
	year={2007},
	publisher={Episciences. org}
}

@article{Angluin87,
  author       = {Dana Angluin},
  title        = {Learning Regular Sets from Queries and Counterexamples},
  journal      = {Inf. Comput.},
  volume       = {75},
  number       = {2},
  pages        = {87--106},
  year         = {1987}
}

@article{somani2013despot,
	title={DESPOT: Online POMDP planning with regularization},
	author={Somani, Adhiraj and Ye, Nan and Hsu, David and Lee, Wee Sun},
	journal={Advances in neural information processing systems},
	volume={26},
	year={2013}
}

@article{silver2010monte,
	title={Monte-Carlo planning in large POMDPs},
	author={Silver, David and Veness, Joel},
	journal={Advances in neural information processing systems},
	volume={23},
	year={2010}
}

@incollection{alur2022framework,
	title={A framework for transforming specifications in reinforcement learning},
	author={Alur, Rajeev and Bansal, Suguman and Bastani, Osbert and Jothimurugan, Kishor},
	booktitle={Principles of Systems Design: Essays Dedicated to Thomas A. Henzinger on the Occasion of His 60th Birthday},
	pages={604--624},
	year={2022},
	publisher={Springer}
}

@article{madani1999undecidability,
	title={On the undecidability of probabilistic planning and infinite-horizon partially observable Markov decision problems},
	author={Madani, Omid and Hanks, Steve and Condon, Anne},
	journal={AAAI},
	volume={10},
	number={315149.315395},
	year={1999}
}

@inproceedings{pineau2003point,
	title={Point-based value iteration: An anytime algorithm for POMDPs},
	author={Pineau, Joelle and Gordon, Geoff and Thrun, Sebastian and others},
	booktitle={Ijcai},
	volume={3},
	pages={1025--1032},
	year={2003}
}

@article{spaan2005perseus,
	title={Perseus: Randomized point-based value iteration for POMDPs},
	author={Spaan, Matthijs TJ and Vlassis, Nikos},
	journal={Journal of artificial intelligence research},
	volume={24},
	pages={195--220},
	year={2005}
}

@inproceedings{kurniawati2008sarsop,
	title={Sarsop: Efficient point-based pomdp planning by approximating optimally reachable belief spaces.},
	author={Kurniawati, Hanna and Hsu, David and Lee, Wee Sun},
	booktitle={Robotics: Science and systems},
	volume={2008},
	year={2008},
	organization={Citeseer}
}

@article{norman2017verification,
	title={Verification and control of partially observable probabilistic systems},
	author={Norman, Gethin and Parker, David and Zou, Xueyi},
	journal={Real-Time Systems},
	volume={53},
	pages={354--402},
	year={2017},
	publisher={Springer}
}

@inproceedings{bork2020verification,
	title={Verification of indefinite-horizon POMDPs},
	author={Bork, Alexander and Junges, Sebastian and Katoen, Joost-Pieter and Quatmann, Tim},
	booktitle={International Symposium on Automated Technology for Verification and Analysis},
	pages={288--304},
	year={2020},
	organization={Springer}
}

@inproceedings{bork2022under,
	title={Under-approximating expected total rewards in POMDPs},
	author={Bork, Alexander and Katoen, Joost-Pieter and Quatmann, Tim},
	booktitle={International Conference on Tools and Algorithms for the Construction and Analysis of Systems},
	pages={22--40},
	year={2022},
	organization={Springer}
}

@inproceedings{bork2024learning,
	title={Learning explainable and better performing representations of POMDP strategies},
	author={Bork, Alexander and Chakraborty, Debraj and Grover, Kush and K{\v{r}}et{\'\i}nsk{\'y}, Jan and Mohr, Stefanie},
	booktitle={International Conference on Tools and Algorithms for the Construction and Analysis of Systems},
	pages={299--319},
	year={2024},
	organization={Springer}
}

@inproceedings{andriushchenko2021paynt,
	title={PAYNT: a tool for inductive synthesis of probabilistic programs},
	author={Andriushchenko, Roman and {\v{C}}e{\v{s}}ka, Milan and Junges, Sebastian and Katoen, Joost-Pieter and Stupinsk{\`y}, {\v{S}}imon},
	booktitle={International Conference on Computer Aided Verification},
	pages={856--869},
	year={2021},
	organization={Springer}
}

@inproceedings{andriushchenko2023search,
	title={Search and explore: Symbiotic policy synthesis in pomdps},
	author={Andriushchenko, Roman and Bork, Alexander and {\v{C}}e{\v{s}}ka, Milan and Junges, Sebastian and Katoen, Joost-Pieter and Mac{\'a}k, Filip},
	booktitle={International Conference on Computer Aided Verification},
	pages={113--135},
	year={2023},
	organization={Springer}
}

@article{STORM22,
	author    = {Christian Hensel and
	Sebastian Junges and
	Joost{-}Pieter Katoen and
	Tim Quatmann and
	Matthias Volk},
	title     = {The probabilistic model checker {Storm}},
	journal   = {Int. J. Softw. Tools Technol. Transf.},
	volume    = {24},
	number    = {4},
	pages     = {589--610},
	year      = {2022}
}

@article{carr2021task,
	title={Task-aware verifiable RNN-based policies for partially observable Markov decision processes},
	author={Carr, Steven and Jansen, Nils and Topcu, Ufuk},
	journal={Journal of Artificial Intelligence Research},
	volume={72},
	pages={819--847},
	year={2021}
}

@article{mccallum1992first,
	title={First results with utile distinction memory for reinforcement learning},
	author={McCallum, R. A.},
    year={1992}
}

@inproceedings{kaelbling1995learning,
	title={Learning policies for partially observable environments: Scaling up},
	author={Kaelbling, Leslie Pack},
	booktitle={Machine Learning Proceedings 1995: Proceedings of the Twelfth International Conference on Machine Learning, Tahoe City, California, July 9-12 1995},
	pages={362},
	year={1995},
	organization={Morgan Kaufmann}
}

@inproceedings{junges2021enforcing,
	title={Enforcing almost-sure reachability in POMDPs},
	author={Junges, Sebastian and Jansen, Nils and Seshia, Sanjit A},
	booktitle={International Conference on Computer Aided Verification},
	pages={602--625},
	year={2021},
	organization={Springer}
}

@article{DBLP:journals/fmsd/BouyerBBR07,
  author       = {Patricia Bouyer and
                  Thomas Brihaye and
                  V{\'{e}}ronique Bruy{\`{e}}re and
                  Jean{-}Fran{\c{c}}ois Raskin},
  title        = {On the optimal reachability problem of weighted timed automata},
  journal      = {Formal Methods Syst. Des.},
  volume       = {31},
  number       = {2},
  pages        = {135--175},
  year         = {2007}
}

@inproceedings{DBLP:conf/formats/BrihayeBR05,
  author       = {Thomas Brihaye and
                  V{\'{e}}ronique Bruy{\`{e}}re and
                  Jean{-}Fran{\c{c}}ois Raskin},
  title        = {On Optimal Timed Strategies},
  booktitle    = {{FORMATS}},
  series       = {Lecture Notes in Computer Science},
  volume       = {3829},
  pages        = {49--64},
  publisher    = {Springer},
  year         = {2005}
}

@inproceedings{prism,
  author       = {Marta Z. Kwiatkowska and
                  Gethin Norman and
                  David Parker},
  title        = {{PRISM} 4.0: Verification of Probabilistic Real-Time Systems},
  booktitle    = {{CAV}},
  series       = {Lecture Notes in Computer Science},
  volume       = {6806},
  pages        = {585--591},
  publisher    = {Springer},
  year         = {2011}
}

@book{baier2008principles,
	title={Principles of model checking},
	author={Baier, Christel and Katoen, Joost-Pieter},
	year={2008},
	publisher={MIT press}
}

@inproceedings{shahbaz2009inferring,
  title={Inferring mealy machines},
  author={Shahbaz, Muzammil and Groz, Roland},
  booktitle={International Symposium on Formal Methods},
  pages={207--222},
  year={2009},
  organization={Springer}
}

@inproceedings{jimenez1999computing,
	title={Computing the k shortest paths: A new algorithm and an experimental comparison},
	author={Jim{\'e}nez, V{\'\i}ctor M and Marzal, Andr{\'e}s},
	booktitle={Algorithm Engineering: 3rd International Workshop, WAE’99 London, UK, July 19--21, 1999 Proceedings 3},
	pages={15--29},
	year={1999},
	organization={Springer}
}

@inproceedings{kocsis2006bandit,
  title={Bandit based monte-carlo planning},
  author={Kocsis, Levente and Szepesv{\'a}ri, Csaba},
  booktitle={European conference on machine learning},
  pages={282--293},
  year={2006},
  organization={Springer}
}

@inproceedings{chatterjee2025value,
  author       = {Krishnendu Chatterjee and
                  Laurent Doyen and
                  Jean{-}Fran{\c{c}}ois Raskin and
                  Ocan Sankur},
  title        = {The Value Problem for Multiple-Environment MDPs with Parity Objective},
  booktitle    = {{ICALP}},
  series       = {LIPIcs},
  volume       = {334},
  pages        = {150:1--150:17},
  publisher    = {Schloss Dagstuhl - Leibniz-Zentrum f{\"{u}}r Informatik},
  year         = {2025}
}

@article{muvskardin2022aalpy,
  title={AALpy: an active automata learning library},
  author={Mu{\v{s}}kardin, Edi and Aichernig, Bernhard K and Pill, Ingo and Pferscher, Andrea and Tappler, Martin},
  journal={Innovations in Systems and Software Engineering},
  volume={18},
  number={3},
  pages={417--426},
  year={2022},
  publisher={Springer}
}

@inproceedings{learn-symb-aut,
  author       = {Samuel Drews and
                  Loris D'Antoni},
  title        = {Learning Symbolic Automata},
  booktitle    = {{TACAS} {(1)}},
  series       = {Lecture Notes in Computer Science},
  volume       = {10205},
  pages        = {173--189},
  year         = {2017}
}

@article{WZL21,
  author       = {Bo Wu and
                  Xiaobin Zhang and
                  Hai Lin},
  title        = {Supervisor synthesis of {POMDP} via automata learning},
  journal      = {Autom.},
  volume       = {129},
  pages        = {109654},
  year         = {2021}
}

@inproceedings{chatterjee2010randomness,
  title={Randomness for free},
  author={Chatterjee, Krishnendu and Doyen, Laurent and Gimbert, Hugo and Henzinger, Thomas A},
  booktitle={International Symposium on Mathematical Foundations of Computer Science},
  pages={246--257},
  year={2010},
  organization={Springer}
}

@article{rivest1989inference,
  author       = {Ronald L. Rivest and
                  Robert E. Schapire},
  title        = {Inference of Finite Automata Using Homing Sequences},
  journal      = {Inf. Comput.},
  volume       = {103},
  number       = {2},
  pages        = {299--347},
  year         = {1993}
}

@article{smith2012heuristic,
  title={Heuristic search value iteration for POMDPs},
  author={Smith, Trey and Simmons, Reid},
  journal={arXiv preprint arXiv:1207.4166},
  year={2012}
}

@misc{artifact,
  author       = {Chakraborty, Debraj and
                  Majumdar, Anirban and
                  Mathew, Prince and
                  Mukherjee, Sayan and
                  Raskin, Jean-Francois},
  title        = {Artifact for "Synthesizing POMDP Policies:
                   Sampling Meets Model-checking via Learning"
                  },
  month        = may,
  year         = 2026,
  publisher    = {Zenodo},
  doi          = {10.5281/zenodo.20084734},
}

@article{PeledVY02,
  author       = {Doron A. Peled and
                  Moshe Y. Vardi and
                  Mihalis Yannakakis},
  title        = {Black Box Checking},
  journal      = {J. Autom. Lang. Comb.},
  volume       = {7},
  number       = {2},
  pages        = {225--246},
  year         = {2002}
}

@inproceedings{ShijuboWS21,
  author       = {Junya Shijubo and
                  Masaki Waga and
                  Kohei Suenaga},
  title        = {Efficient Black-Box Checking via Model Checking with Strengthened
                  Specifications},
  booktitle    = {{RV}},
  series       = {Lecture Notes in Computer Science},
  pages        = {100--120},
  publisher    = {Springer},
  year         = {2021}
}

@inproceedings{labbaf2023compositional,
  title={Compositional learning for interleaving parallel automata},
  author={Labbaf, Faezeh and Groote, Jan Friso and Hojjat, Hossein and Mousavi, Mohammad Reza},
  booktitle={International Conference on Foundations of Software Science and Computation Structures},
  pages={413--435},
  year={2023},
  organization={Springer}
}

@inproceedings{fujinami2025componentwise,
  title={Componentwise Automata Learning for System Integration},
  author={Fujinami, Hiroya and Waga, Masaki and An, Jie and Suenaga, Kohei and Yanagisawa, Nayuta and Iseri, Hiroki and Hasuo, Ichiro},
  booktitle={International Symposium on Automated Technology for Verification and Analysis},
  pages={3--26},
  year={2025},
  organization={Springer}
}

@inproceedings{neele2023compositional,
  title={Compositional automata learning of synchronous systems},
  author={Neele, Thomas and Sammartino, Matteo},
  booktitle={International Conference on Fundamental Approaches to Software Engineering},
  pages={47--66},
  year={2023},
  organization={Springer}
}

\newpage
\appendix

\section{Missing Proofs from \Cref{sec:problem-formulation}}
\label{app:problem-formulation}

\finFromW*
\begin{proof}
    Winning a safety objective with probability $1$ can only be achieved by a policy that entirely avoids bad states. Indeed, if any bad state is reached, it is reached after a finite execution prefix, and this prefix has non-zero probability, which would contradict the fact that the policy is winning almost surely. Thus, the stochastic nodes can be viewed as being controlled by an adversary in a zero-sum two-player game with partial observation. We know that in such games, finite memory (exponential in the worst case) suffices~\cite{chatterjee2007algorithms}.\qed
\end{proof}

\safeORinW*
\begin{proof}
    Without loss of any generality, assume that we have a single sink state with unique observation ${\sf Bad}$, and let $b_{\sf Bad}$ be the belief that assigns probability $1$ to the bad sink state.

    Consider any belief $b\notin W$. 
    By definition of $W$, for any policy $\sigma$, we have $\mathbb{P}^{\sigma}_{\belief_{\pomdp}} (\Box\neg {\sf Bad} \mid \text{initial belief } b) < 1$. 
    Then, from $b$ under policy $\sigma$, there must exist a finite path of size $\leq |\mdpstates|$ from $b$ to $b_{\sf Bad}$ 
    (this path corresponds to the simple path that starts from a state in the support of $b$ and reaches the bad sink state).
    Thus, for all $b\in W$, we have $\mathbb{P}^{\sigma}_{\belief_{\pomdp}} (\Diamond^{\leq |\mdpstates|} {\sf Bad} \mid \text{initial belief } b) \geq \eta$, 
    where $\eta=(p_{\min})^{|\mdpstates|}$, where $p_{\min}$ is the minimum non-zero transition probability in the underlying MDP.
    Therefore, for any belief $b\notin W$, we have $\mathbb{P}^{\sigma}_{\belief_{\pomdp}} (\Box^{\leq |\mdpstates|}\neg {\sf Bad}) < 1 - \eta$. 

    Now if the initial belief is in $W$, then $\mathbb{P}^{\sigma}_{\belief_{\pomdp}} (\Box\neg {\sf Bad} \wedge \Box\neg W)=0$ trivially holds.
    Otherwise, consider the probability of staying safe and not reaching $W$ within $k\cdot |\mdpstates|$ steps, for some $k \in \mathbb{N}$:
    $\mathbb{P}^{\sigma}_{\belief_{\pomdp}} (\Box^{\leq k\cdot|\mdpstates|}\neg {\sf Bad}\wedge \Box^{\leq k\cdot|\mdpstates|}\neg W) < (1 - \eta)^k\,.$
    Now, if $k$ tends to infinity, we have $\mathbb{P}^{\sigma}_{\belief_{\pomdp}} (\Box\neg {\sf Bad} \wedge \Box\neg W)=0$.\qed
\end{proof}

\safetyDrewEq*
\begin{proof}
\begin{enumerate}
\item[1.] Recall that, for POMDPs $\pomdp$, $\pomdpr$, 
    \[\prps(\Box \neg {\sf Bad})=1+\lim_{\lambda \rightarrow 1} \expsp({\sf \DRew^{\lambda}}) 
    = \lim_{\lambda \rightarrow 1}(1 + \expsp({\sf \DRew^{\lambda}}))\]
    From the above equation, we get that as $\lambda \rightarrow 1$, $(1 + \mathbb{E}^{\sigma}_{\pomdpr}({\DRew^{\lambda}})) \rightarrow \mathbb{P}^{\sigma}_{M}(\Box \neg {\sf Bad})$.
    Then, we can deduce that for every $\varepsilon > 0$, there exists some $\delta \geq 0$ such that whenever $|1 - \lambda| < \delta$, $|(1 + \mathbb{E}^{\sigma}_{\pomdpr}({\DRew^{\lambda}})) - (\mathbb{P}^{\sigma}_{\pomdp}(\Box \neg {\sf Bad}))| < \varepsilon$.
    Therefore, for every $\varepsilon > 0$, we can always choose a discount factor $\lambda$ \emph{sufficiently}-close to $1$, for which  $|(1 + \mathbb{E}^{\sigma}_{\pomdpr}({\DRew^{\lambda}})) - (\mathbb{P}^{\sigma}_{\pomdp}(\Box \neg {\sf Bad}))| < \varepsilon$ holds. This implies that if $(1 + \mathbb{E}^{\sigma}_{\pomdpr}({\DRew^{\lambda}})) > \threshold$ for some $\threshold$, then $\mathbb{P}^{\sigma}_{\pomdp}(\Box \neg {\sf Bad}) > \threshold - \varepsilon$.
    \item[2.] The proof of the second statement is analogous.\qed
\end{enumerate}
\end{proof}

\section{Missing Proofs from~\Cref{sec:cplus}}
\label{app:cplus}

\counterex*
\begin{proof}
Towards a contradiction, assume that there exists a policy $\sigma \in \tsp$ such that for every $\rho \in \ce$, we have $\rho \models \sigma$. Since $\ce$ is a counterexample, we have $\cylinder(\ce) \subseteq \Diamond \bad$ and $\prph(\cylinder(\ce)) >1-\threshold$. 
This implies $\prps(\Diamond \bad) > 1 - \threshold$, and hence
$\prps(\Box \neg \bad) \le \threshold$. This contradicts $\sigma \in \mathcal{G}$. 
\end{proof}

\relCompl*

\begin{proof}
Correctness trivially follows since if \cplus terminates, then it does so with a successful model-checking query.
    For the relative completeness,
    let $\hypothesis_{\sigma}$ be the FSC corresponding to $\sigma$ according to which \despot is suggesting actions. 
    Since action queries are performed according to $\hypothesis_{\sigma}$, the observation table $\ObsTable$ maintained by \cplus is always consistent with $\hypothesis_{\sigma}$. 
    Now, let $\hypothesis_i$ be the hypothesis generated at the $i^{th}$ iteration of \cplus, and let $\ce$ be the counterexample produced by \storm at that iteration. Note that, by~\cref{prop:safety-drew-eq}, $\hypothesis_{\sigma}$ satisfies the safety specification in $\pomdp$.
    Therefore, by~\Cref{lem:eq-counterexample}, there exists $\rho \in \ce$ such that 
    $\rho \not\models \sigma$.
    Now, since $\rho \models \sigma_{\hypothesis_i}$, $\rho$ must have a finite prefix $\rho_h$ such that $\sigma_{\hypothesis_i}(\rho_h) \neq \sigma(\rho_h)$. 
    Since $\ObsTable$ is consistent with $\hypothesis_{\sigma}$, 
    adding $\rho_h$ to $\ObsTable$ rules out this counterexample from subsequent model-checking queries.
    The termination of \cplus then follows from the termination of  $L^*$  for Mealy machines~\cite{shahbaz2009inferring}. \qed
\end{proof}

\begin{table}[t]
\centering
\tiny
\begin{tabular}{l!{\vrule width 1.2pt}c|c|c!{\vrule width 1.2pt}c|c|c!{\vrule width 1.2pt}c|c|c!{\vrule width 1.2pt}c|c!{\vrule width 1.2pt}c|c|c}
\toprule
\multirow{2}{*}{Model} 
& \multirow{2}{*}{$|\mdpstates|$} 
& \multirow{2}{*}{$|\observations|$} 
& \multirow{2}{*}{$\alpha$}
& \multicolumn{3}{c!{\vrule width 1.2pt}}{$\cplusds$} 
& \multicolumn{3}{c!{\vrule width 1.2pt}}{$\cplusss$} 
& \multicolumn{2}{c!{\vrule width 1.2pt}}{$\paynt$}
& \multicolumn{3}{c}{$\storm$} \\ [2pt]
\cmidrule{5-15}

& & & 
& $|\hypothesis|$ & iter & time
& $|\hypothesis|$ & iter & time 
& $|\hypothesis|$ & time
& $|\belief|$ & value & time\\
\hline
\multirow{2}{*}{} & \multirow{2}{*}{} & \multirow{2}{*}{} &  
& & &  
& & &
& & & 
& \multirow{2}{*}{} & \multirow{2}{*}{} \\[-6pt]
hallway & 16 & 3 & 0.99 & 25 & 3 & 78.11 & 24 & 3 & 49.32 & 5 & \textbf{1.44} & 3 & 1.00 & 0.01\\
hallway-2 & 24 & 3 & 0.99 & 20 & 4 & \textbf{68.99} & 20 & 3 & 43.41 & - & - & 3 & 1.00 & 0.01\\
hallway-simple-50 & 154 & 4 & 0.99 & 53 & 2 & 138.77 & 53 & 2 & \textbf{22.61} & - & - & 156 & 1.00 & 1.00\\
\midrule\midrule
cards-removed-2-bounded & 9 & 6 & 0.5 & 4 & 1 & 1.59 & 4 & 1 & 1.48 & 1 & \textbf{0.01} & 10 & 1.00 & 0.01\\
cards-removed-3-bounded & 76 & 10 & 0.5 & 7 & 1 & 2.10 & 20 & 2 & 4.04 & 1 & \textbf{0.02} & 103 & 0.97 & 1.00\\
cards-removed-4-bounded & 165 & 12 & 0.5 & 48 & 2 & 12.28 & 29 & 2 & 7.42 & 2 & \textbf{0.16} & 322 & 0.91 & 2.00\\
cards-removed-5-bounded & 306 & 14 & 0.5 & 118 & 2 & 44.74 & 47 & 2 & \textbf{16.13} & - & - & 919 & 0.85 & 8.00\\
\midrule\midrule
cards-added-2-bounded & 35 & 8 & 0.5 & 7 & 1 & 1.93 & 7 & 1 & 1.77 & 1 & \textbf{0.01} & 40 & 0.74 & 0.01\\
cards-added-3-bounded & 94 & 10 & 0.5 & 29 & 4 & 11.80 & 21 & 5 & 10.66 & 2 & \textbf{0.03} & 320 & 0.62 & 1.00\\
cards-added-4-bounded & 197 & 12 & 0.5 & 356 & 9 & 224.87 & 108 & 19 & \textbf{184.11} & - & - & 2526 & 0.56 & 0.01\\ 
cards-added-5-bounded & 356 & 14 & 0.5 & - & - & - & - & - & - & - & - & 19603 & 0.51 & 93.00\\ 
\midrule\midrule
cards-removed-2-unbounded & 13 & 9 & 0.5 & 6 & 1 & 2.74 & 6 & 1 & 2.36 & 1 & \textbf{0.01} & 11 & 0.95 & 0.01\\
 &  &  & 0.2 & 6 & 1 & 2.77 & 6 & 1 & 2.34 & 1 & \textbf{0.01} &  &  & \\
\cmidrule{4-15}
 cards-removed-3-unbounded & 25 & 11 & 0.5 & 15 & 1 & 8.29 & 15 & 1 & 2.49 & 2 & \textbf{2.33} & 32 & 0.85 & 0.01\\
 &  &  & 0.2 & 15 & 1 & 8.30 & 15 & 1 & 2.42 & 1 & \textbf{0.02} &  &  & \\
\cmidrule{4-15}
 cards-removed-4-unbounded & 41 & 13 & 0.5 & 27 & 1 & 13.67 & 21 & 1 & \textbf{4.32} & - & - & 79 & 0.74 & 0.01\\
 &  &  & 0.2 & 27 & 1 & 13.68 & 21 & 1 & 4.35 & 1 & \textbf{0.18} &  &  & \\
\cmidrule{4-15}
 cards-removed-5-unbounded & 61 & 15 & 0.5 & 337 & 4 & 545.20 & 31 & 1 & \textbf{7.75} & - & - & 190 & 0.64 & 0.01\\
 &  &  & 0.2 & 40 & 1 & 28.20 & 31 & 1 & 7.85 & 1 & \textbf{0.01} &  &  & \\
\cmidrule{4-15}
 cards-removed-6-unbounded & 85 & 17 & 0.5 & - & - & - & 63 & 2 & \textbf{39.79} & - & - & 445 & 0.55 & 3.00\\
 &  &  & 0.2 & 34 & 1 & 36.27 & 43 & 1 & \textbf{14.92} & 2 & 188.56 &  &  & \\
\cmidrule{4-15}
 cards-removed-7-unbounded & 113 & 19 & 0.5 & - & - & - & - & - & - & - & - & 1020 & 0.47 & 7.00\\
 &  &  & 0.2 & 217 & 2 & 480.31 & 58 & 1 & \textbf{59.55} & - & - &  &  & \\
\cmidrule{4-15}
 cards-removed-8-unbounded & 145 & 21 & 0.5 & - & - & - & - & - & - & - & - & 2299 & 0.40 & 18.00\\
 &  &  & 0.2 & - & - & - & 60 & 1 & \textbf{88.03} & - & - &  &  & \\
\midrule\midrule
cards-added-2-unbounded & 17 & 9 & 0.5 & 6 & 1 & 5.11 & 7 & 1 & 2.10 & 1 & \textbf{0.01} & 417 & 0.78 & 1.00\\
 &  &  & 0.2 & 6 & 1 & 5.11 & 7 & 1 & 2.13 & 1 & \textbf{0.01} &  &  & \\
 \cmidrule{4-15}
cards-added-3-unbounded & 31 & 11 & 0.5 & 12 & 1 & \textbf{8.74} & 15 & 2 & 11.64 & - & - & 47908 & 0.64 & 183.00\\
 &  &  & 0.2 & 12 & 1 & 8.75 & 11 & 1 & 5.05 & 1 & \textbf{0.04} &  &  & \\
  \cmidrule{4-15}
cards-added-4-unbounded & 49 & 13 & 0.5 & - & - & - & - & - & - & - & - & 4402982 & 0.54 & 28.88\\
 &  &  & 0.2 & 17 & 1 & 10.03 & 19 & 1 & 8.89 & 1 & \textbf{0.44} &  &  & \\
  \cmidrule{4-15}
cards-added-5-unbounded & 71 & 15 & 0.5 & - & - & - & - & - & - & - & - & 10801995 & 0.47 & 217.66\\
 &  &  & 0.2 & 28 & 1 & 23.43 & 28 & 1 & 20.18 & 1 & \textbf{2.28} &  &  & \\
 \cmidrule{4-15}
 cards-added-6-unbounded & 97 & 17 & 0.5 & - & - & - & - & - & - & - & - & 9507680 & 0.40 & 161.72\\
 &  &  & 0.2 & 39 & 1 & 45.89 & 38 & 1 & 39.68 & 1 & \textbf{0.01} &  &  & \\
\cmidrule{4-15}
 cards-added-7-unbounded & 127 & 19 & 0.5 & - & - & - & - & - & - & - & - & 9583831 & 0.32 & 159.36\\
 &  &  & 0.2 & 53 & 1 & 86.34 & 52 & 1 & 76.75 & 1 & 0.01 &  &  & \\
\cmidrule{4-15}
 cards-added-8-unbounded & 161 & 21 & 0.5 & - & - & - & - & - & - & - & - & 9236973 & 0.25 & 152.18\\
 &  &  & 0.2 & 68 & 1 & 148.85 & 68 & 1 & \textbf{141.14} & - & - &  &  & \\
 \midrule\midrule 
cheese-maze & 15 & 8 & 0.5 & 14 & 1 & 7.95 & 12 & 1 & 2.80 & 1 & \textbf{0.01} & 17232927 & 0.99 & 231.64\\
 &  &  & 0.8 & 14 & 1 & 6.46 & 12 & 1 & 2.80 & 1 & \textbf{0.01} &  &  & \\
 \cmidrule{4-15}
cheese-maze-det & 15 & 8 & 0.5 & 9 & 1 & 6.11 & 9 & 1 & 1.52 & 2 & \textbf{0.01 }& 22 & 1.00 & 0.01\\
 &  &  & 0.8 & 9 & 1 & 6.10 & 9 & 1 & 1.52 & 2 & \textbf{0.01} &  &  & \\
 \cmidrule{4-15}
refuel($N=6,E=8$) & 270 & 36 & 0.5 & 3 & 1 & 4.97 & 3 & 1 & 1.46 & 1 & \textbf{0.07} & 26423672 & 0.72 & 288.55\\
 &  &  & 0.8 & - & - & - & - & - & - & 1 & \textbf{0.07} &  &  & \\
 \cmidrule{4-15}
rocks($N=4$) & 332 & 66 & 0.5 & 12 & 2 & 32.62 & 26 & 1 & 15.60 & 1 & \textbf{0.05} & 485 & 1.00 & 7.00\\
 &  &  & 0.8 & - & - & - & 28 & 2 & 34.56 & 1 & \textbf{0.05} &  &  & \\
 \bottomrule
\end{tabular}
\caption{We report here various statistics regarding three tools: \cplusds, \cplusss, \storm, and \paynt. Bold entries indicate the best-performing tool (in terms of time) among those that produce verified FSCs.} 
\label{tab:other-benchmarks-full}
\end{table}

\clearpage
\section{Example: Cards}
\label{app:cards}
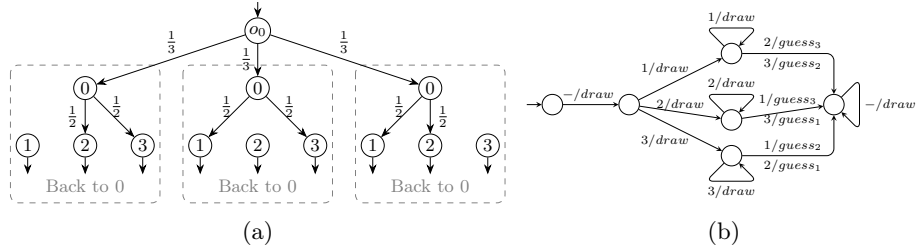
\begin{figure}[h]
    \centering
    \begin{subfigure}{0.55\linewidth}
        \centering
        \resizebox{\linewidth}{!}{
        \begin{tikzpicture}[
  vertex/.style = {circle, draw, minimum size=4mm, inner sep=1pt},
  final/.style = {double, double distance=1pt},
  ->, >=Stealth
  ]
\small
\node[vertex] (0) at (0,1) {$\init$};
\draw [dashed, rounded corners, draw opacity = 0.5] (-4.3,0.4) rectangle (-1.7, -2);
\draw [dashed, rounded corners, draw opacity = 0.5] (-1.3,0.4) rectangle (1.3,-2);
\draw [dashed, rounded corners, draw opacity = 0.5] (1.7,0.4) rectangle (4.3,-2);

\node[vertex] (M1) at (-3,0) {$0$};

\node[vertex] (M11) at (-4,-1) {$1$};
\node[vertex] (M12) at (-3,-1) {$2$};
\node[vertex] (M13) at (-2,-1) {$3$};

\node[vertex] (M2) at (0,0) {$0$};

\node[vertex] (M21) at (-1,-1) {$1$};
\node[vertex] (M22) at (0,-1) {$2$};
\node[vertex] (M23) at (1,-1) {$3$};

\node[vertex] (M3) at (3,0) {$0$};

\node[vertex] (M31) at (2,-1) {$1$};
\node[vertex] (M32) at (3,-1) {$2$};
\node[vertex] (M33) at (4,-1) {$3$};

\node[opacity=0.5] at (-3,-1.7) {Back to $0$};
\node[opacity=0.5] at (0,-1.7) {Back to $0$};
\node[opacity=0.5] at (3,-1.7) {Back to $0$};


\draw[->] (0,1.5) -- (0);

\draw[->] (0) -- (M1);
\draw[->] (0) -- (M2);
\draw[->] (0) -- (M3);

\node at (-1.5, 0.8) {$\frac{1}{3}$};
\node at (-0.2, 0.5) {$\frac{1}{3}$};
\node at (1.5, 0.8) {$\frac{1}{3}$};

\draw[->] (M1) -- (M12) node[midway,left] {$\frac{1}{2}$};
\draw[->] (M1) -- (M13) node[near start,right] {$\frac{1}{2}$};

\draw[->] (M2) -- (M21) node[near start,left] {$\frac{1}{2}$};
\draw[->] (M2) -- (M23) node[near start,right] {$\frac{1}{2}$};

\draw[->] (M3) -- (M31) node[near start,left] {$\frac{1}{2}$};
\draw[->] (M3) -- (M32) node[midway,right] {$\frac{1}{2}$};

\draw[->] (M11) -- (-4,-1.5);
\draw[->] (M12) -- (-3,-1.5);
\draw[->] (M13) -- (-2,-1.5);

\draw[->] (M21) -- (-1,-1.5);
\draw[->] (M22) -- (0,-1.5);
\draw[->] (M23) -- (1,-1.5);

\draw[->] (M31) -- (2,-1.5);
\draw[->] (M32) -- (3,-1.5);
\draw[->] (M33) -- (4,-1.5);

\end{tikzpicture}}
        \caption{}
        \label{fig:card-removed-3}
    \end{subfigure}
    \hfill
    \begin{subfigure}{0.44\linewidth}
        \centering
        \resizebox{\linewidth}{!}{
        \begin{tikzpicture}
    [state/.style={draw, circle, inner sep=2pt, minimum size=4mm}]
    
    \begin{scope}[every node/.style={state}]
      \node (p0) at (-0.5, 0) {}; 
      \node (p1) at (1, 0) {}; 
      \node (p21) at (3, 1) {}; 
      \node (p22) at (3, -0.3) {}; 
      \node (p23) at (3, -1) {}; 
      \node (p3) at (5, 0) {};   
    \end{scope}

    \begin{scope}[->, >=stealth]
      \draw (-1, 0) to (p0); 
      \draw (p0) to (p1); 
      \draw (p1) to (p21); 
      \draw (p1) to (p22); 
      \draw (p1) to (p23); 
      \draw [rounded corners] (p21) to (5,1) to (p3); 
      \draw [rounded corners] (p21) to (2.5, 1.5) to (3.5, 1.5) to (p21); 
      \draw (p22) to (p3); 
      \draw [rounded corners] (p22) to (2.5, 0.2) to (3.5, 0.2) to (p22);
      \draw [rounded corners] (p23) to (5,-1) to (p3); 
      \draw [rounded corners] (p23) to (2.5, -1.5) to (3.5, -1.5) to (p23);
      \draw [rounded corners] (p3) to (5.5, 0.5) to (5.5, -0.5) to (p3);
    \end{scope}	

    \node at (0.2, 0.2) {\scriptsize $- / draw$}; 
    \node at (1.7, 0.7) {\scriptsize $1 / draw$}; 
    \node at (2, 0) {\scriptsize $2 / draw$}; 
    \node at (3, 1.7) {\scriptsize $1 / draw$};
    \node at (3, 0.4) {\scriptsize $2 / draw$};
    \node at (3, -1.7) {\scriptsize $3 / draw$};         
    \node at (1.7, -0.7) {\scriptsize $3 / draw$}; 
    \node at (4.2, 1.2) {\scriptsize $2 / guess_3$}; 
    \node at (4.2, 0.8) {\scriptsize $3 / guess_2$}; 
    \node at (4.1, 0.1) {\scriptsize $1 / guess_3$}; 
    \node at (4.2, -0.3) {\scriptsize $3 / guess_1$}; 
    \node at (4.2, -0.8) {\scriptsize $1 / guess_2$}; 
    \node at (4.2, -1.2) {\scriptsize $2 / guess_1$}; 
    \node at (6.1, 0) {\scriptsize $- / draw$};

    \end{tikzpicture}}
        \caption{}
        \label{fig:policy-for-card-removed}
    \end{subfigure}
    \caption{A POMDP representing~\Cref{ex:3cards} is depicted in~\cref{fig:card-removed-3}. Labels on states denote their corresponding observations. Only the transitions for the ${\sf draw}$ action from each state are depicted. The arrows from the terminal states in each block represent transitions back to the observation $0$ of the {corresponding} block.
    \cref{fig:policy-for-card-removed} represents a winning policy for the threshold-safety policy for {this} POMDP.}
    \label{fig:card-removed-3-and-fsc}
\end{figure}

\begin{example}[Removing Cards-$3$]
\label{ex:3cards}
Consider the following card example, adapted from~\cite{chatterjee2025value}, depicted in \Cref{fig:card-removed-3}. We have a $3$-card deck, numbered $1,2,3$, from which one card, initially chosen at random with probability $\frac{1}{3}$, is missing. The possible actions are: ${\sf draw}$, and $({\sf guess}_{i})_{i\in \{1,2,3\}}$. Only the transitions corresponding to the ${\sf draw}$ action from each state are depicted in the figure. The observations are: the numbers on the cards, the initial observation $\init$, and a dummy observation $0$. Thus, the agent does not directly observe which block 
(\emph{i.e.}, which missing-card configuration) it is in, but instead observes the cards drawn (at random) from the deck.
From the bottom of each block, if the ${\sf draw}$ action is chosen, the system returns to the root of the corresponding block (represented as ``Back to $0$''). In the $i$-th block, from the $0$ observation, on action ${\sf draw}$, there is a transition to each of the present cards with probability $\frac{1}{2}$. When the agent becomes confident about which card is missing, it can play the corresponding ${\sf guess}_i$. Guessing the correct card 
leads to a winning state, while choosing it incorrectly leads to a separate sink state (these two states are not shown in the figure). The objective of the agent is to correctly guess the missing card.

\end{example}

    Figure~\ref{fig:policy-for-card-removed} depicts a policy (as an FSC) for Example~\ref{ex:3cards}. 
    After the first card is drawn (and hence the block is determined) this policy waits (by repeatedly suggesting to ${\sf draw}$) while the same card is shown. 
    As soon as it observes a different card it can then {correctly guess} the missing card.
    Note that, this policy wins the game with probability $1$. 
\end{document}